\documentclass[journal, onecolumn]{IEEEtran}

\usepackage{amsmath}
\usepackage{amsfonts}
\usepackage{amssymb}
\usepackage{mathtools}
\usepackage{graphicx}%
\usepackage{booktabs}
\usepackage[linesnumbered, commentsnumbered, ruled, vlined]{algorithm2e}

\SetCommentSty{mycommfont}

\providecommand{\U}[1]{\protect\rule{.1in}{.1in}}
\providecommand{\U}[1]{\protect\rule{.1in}{.1in}}
\newtheorem{theorem}{Theorem}

\newtheorem{definition}[theorem]{Definition}
\newtheorem{example}[theorem]{Example}

\newtheorem{lemma}[theorem]{Lemma}
\newtheorem{notation}[theorem]{Notation}

\newtheorem{proposition}[theorem]{Proposition}

\newenvironment{proof}[1][Proof]{\noindent\textbf{#1.} }{\ \rule{0.5em}{0.5em}}

\begin{document}

\title{The $\mathcal{F}^{A}$ Quantifier Fuzzification Mechanism: analysis of convergence and efficient implementations.}

\author{
F\'{e}lix D\'{\i}az-Hermida, Marcos Matabuena, Juan C. Vidal%
\thanks{F. D\'{\i}az-Hermida, M. Matabuena, and J.C. Vidal are with the Centro Singular de Investigaci\'{o}n en Tecnolox\'{\i}as da Informaci\'{o}n (CiTIUS), Universidade de Santiago de Compostela, 15782, Santiago de Compostela, SPAIN.}%
}

%

\maketitle

\begin{abstract}
The fuzzy quantification model $\mathcal{F}^{A}$ has been identified as one of
the best behaved quantification models in several revisions of the field of
fuzzy quantification. This model is, to our knowledge, the unique one
fulfilling the strict \textit{Determiner Fuzzification Scheme} axiomatic
framework that does not induce the standard min and max operators. The main
contribution of this paper is the proof of a convergence result that links this
quantification model with the Zadeh's model when the size of the input sets
tends to infinite. The convergence proof is, in any case, more general than
the convergence to the Zadeh's model, being applicable to any quantitative
quantifier. In addition, recent revisions papers have presented some doubts about
the existence of suitable computational implementations to evaluate the $\mathcal{F}^{A}$ model in practical applications. In order to prove that this model is not
only a theoretical approach, we show exact algorithmic solutions for the most
common linguistic quantifiers as well as an approximate implementation by means
of Monte Carlo. Additionally, we will also give a general overview of the main
properties fulfilled by the $\mathcal{F}^{A}$ model, as a single compendium
integrating the whole set of properties fulfilled by it has not been
previously published.

\end{abstract}


%
\IEEEpeerreviewmaketitle

\section{Introduction}
\IEEEPARstart{A}{} great range of models have been proposed for the evaluation of fuzzy
quantified sentences, being
\cite{Barro02,Delgado00,Delgado14,Sanchez16,DiazHermida00,DiazHermida02-FuzzySets,DiazHermida04IPMU,DiazHermida10Arxiv,DiazHermida04-IEEE,Dubois85,
DVORAK2014,
Glockner03-Generalized,Glockner06Libro,Liu98,Ming2006,Ralescu95,Yager83,Yager88,Yager2016,Zadeh83}
only an example. Several revision papers have also been published, being
\cite{Delgado14} possibly the one that makes a more exhaustive comparison.
Other revision works is worth to mention are
\cite{Barro02,Delgado00,Glockner06Libro,DiazHermida06Tesis, Yager2016}. There
also exists an specific paper \cite{DiazHermida17-FuzzySets}, comparing the
models following the quantification framework presented in
\cite{Glockner06Libro}.

Moreover, fuzzy quantifiers have been used in a wide range of applications
like fuzzy control, temporal reasoning, fuzzy databases, information
retrieval, multi-criteria decision making, data fusion, natural language
generation, etc. In \cite{Delgado14} a list of the main applications of fuzzy
quantifiers is presented.

This paper is devoted to present some relevant new results of the
$\mathcal{F}^{A}$ \textit{quantification model} proposed in
\cite{DiazHermida04-IEEE,DiazHermida04IPMU}. The main result we will present
is a convergence proof that, in the particular case of proportional
quantifiers, assures the convergence of the $\mathcal{F}^{A}$ model to the
\textit{Zadeh's quantification model} \cite{Zadeh83} when the intersection of
fuzzy sets is modelled by means of the probabilistic \textit{tnorm} operator.
Moreover, although several revisions of the fuzzy quantification field
\cite{Delgado14,Yager2016,DiazHermida17-FuzzySets} have presented this model
as one of the best fuzzy quantification models available, some doubts persist
about the possibility of efficiently implementing it
\cite{Delgado14,Yager2016}. For this reason, we will also provide efficient
computational implementations of the $\mathcal{F}^{A}$ quantification model
for the most common linguistic quantifiers as well as the explanation of how
to extend these implementations to other types of linguistic quantifiers.

The definition of the $\mathcal{F}^{A}$ quantification model follows the
Gl\"{o}ckner's approximation to fuzzy quantification \cite{Glockner06Libro}
instead of the common one based on type I, type II quantified expressions
proposed by Zadeh \cite{Zadeh83}. Gl\"{o}ckner's approximation generalizes the
concept of generalized classic quantifier \cite{Barwise81} (second order
predicates or set relationships) to the fuzzy case as fuzzy relationships
between fuzzy sets. Following this idea he recasts the problem of evaluating
fuzzy quantified expressions as a problem of searching for adequate mechanisms to
transform semi-fuzzy quantifiers (specification means) into fuzzy quantifiers
(operational means). The author denominates these transformation mechanisms
\textit{Quantifier Fuzzification Mechanism} (\textit{QFMs}). The followed
approach also generalizes the \textit{Theory of Generalized Quantifiers}
(\textit{TGQ}), that deals with the analysis and modelling of the phenomena of
quantification in natural language, \cite{Barwise81} to the fuzzy case.

In his proposal the author also defined a rigorous axiomatic framework to
ensure the good behavior of QFMs. Models fulfilling this strict framework are
denominated \textit{Determiner Fuzzification Schemes} (\textit{DFSs}) and they
comply with a broad set of properties that guarantee a good behavior from a
linguistic and fuzzy point of view. In \cite{Sanchez16} or
\cite{Glockner06Libro} can be consulted a comparison between Zadeh's and
Gl\"{o}ckner's approaches.

The $\mathcal{F}^{A}$ model is a QFM which fulfills the strict axiomatic
framework proposed by Gl\"{o}ckner, which makes it a DFS. It is to
our knowledge the unique \textit{non-standard DFS} (a DFS\ non inducing the
standard max and min operators).

The structure of this paper is the following. First, we will introduce the
fuzzy quantification framework proposed by Gl\"{o}ckner. Second, we will
present two alternative definitions of the $\mathcal{F}^{A}$ QFM, one based on
the use of fuzzy connectives and the other based on a probabilistic
interpretation of fuzzy sets. After that, we will give a brief review of the
main properties fulfilled by the $\mathcal{F}^{A}$ QFM. The following section
is dedicated to introduce the convergence results of the $\mathcal{F}^{A}$ QFM,
that as a particular case includes the convergence to the Zadeh's model.
However, the rate of convergence is too slow to be of utility in most
applications, but thanks to this result is argued that we can expect very
good approximations of the $\mathcal{F}^{A}$ QFM by means of Monte Carlo. The
final section is devoted to present exact and approximate algorithmic implementations.

\section{The fuzzy quantification framework}

In the specification of the fuzzy quantification framework
\cite{Glockner06Libro}, the author rewrote the problem of defining fuzzy
quantification models as the problem of looking for adequate means to
convert \textit{semi-fuzzy quantifiers }(i.e., mechanisms adequate to specify
the meaning of linguistic quantifiers) into \textit{fuzzy quantifiers} (i.e.,
operational means adequate to apply semi-fuzzy quantifiers to fuzzy inputs).

In this framework, fuzzy quantifiers are just a generalization of crisp
quantifiers to fuzzy sets. We will show below the definition of classic
quantifiers conforming to TGQ.

\begin{definition}
A two valued (generalized) quantifier on a base set $E\neq\varnothing$ is a
mapping $Q:\mathcal{P}\left(  E\right)  ^{n}\longrightarrow\mathbf{2}$, where
$n\in\mathbb{N}$ is the arity (number of arguments) of $Q$, $\mathbf{2}%
=\left\{  0,1\right\}  $ denotes the set of crisp truth values, and
$\mathcal{P}\left(  E\right)  $ is the powerset of $E$.
\end{definition}

Below we show two examples of classic quantifiers:
\begin{align*}
& \mathbf{all}\left(  Y_{1},Y_{2}\right) = Y_{1}\subseteq Y_{2}\\
& \mathbf{at\_least\_60\%}\left(  Y_{1},Y_{2}\right)  
 \qquad = \left\{
\begin{array}
[c]{cc}%
\frac{\left\vert Y_{1}\cap Y_{2}\right\vert }{\left\vert Y_{1}\right\vert
}\geq0.60 & Y_{1}\neq\varnothing\\
1 & Y_{1}=\varnothing
\end{array}
\right.
.\end{align*}

From here on, we will denote $\left\vert E\right\vert =m$.

A fuzzy quantifier assigns a fuzzy value to each possible choice of
$X_{1},\ldots,X_{n}\in\widetilde{\mathcal{P}}\left(  E\right)  $, where by
$\widetilde{\mathcal{P}}\left(  E\right)  $ we are denoting the fuzzy powerset
of $E$.

\begin{definition}
\cite[definition 2.6]{Glockner06Libro} An n-ary fuzzy quantifier
$\widetilde{Q}$ on a base set $E\neq\varnothing$ is a mapping $\widetilde
{Q}:\widetilde{\mathcal{P}}\left(  E\right)  ^{n}\longrightarrow
\mathbf{I=}\left[  0,1\right]  $.
\end{definition}

The next example shows a possible definition of the fuzzy quantifier
$\widetilde{\mathbf{all}}:\widetilde{\mathcal{P}}\left(  E\right)
^{2}\longrightarrow\mathbf{I}$:
\[
\widetilde{\mathbf{all}}\left(  X_{1},X_{2}\right)  =\inf\left\{  \max\left(
1-\mu_{X_{1}}\left(  e\right)  ,\mu_{X_{2}}\left(  e\right)  \right)  :e\in
E\right\}
\]
where by $\mu_{X}\left(  e\right)  $ we are denoting the membership function
of $X\in\widetilde{\mathcal{P}}\left(  E\right)  $.

Previous definition of the fuzzy quantifier $\widetilde{\mathbf{all}}$ seems
plausible. However, the reader could think of in other possible plausible
expressions to model $\widetilde{\mathbf{all}}$ by simply changing the
\textit{tconorm} operator \textit{max} by other \textit{tconorm} operator. For
other quantifiers, like \textit{`at least sixty percent'}, the problem of
establishing adequate models is far from obvious.

In the search of possible solutions for defining fuzzy quantifiers, in
\cite{Glockner06Libro} the concept of semi-fuzzy quantifier was introduced to
work as a `middle point' between classic and fuzzy quantifiers. Semi-fuzzy
quantifiers are close but more powerful than the Zadeh's concept of linguistic
quantifiers \cite{Zadeh83}. Semi-fuzzy quantifiers only accept crisp
arguments, as classic quantifiers, but they have a fuzzy output, as in the
case of fuzzy quantifiers. Semi-fuzzy quantifiers are adequate to capture the
semantics of linguistic quantified expressions.

\begin{definition}
\cite[definition 2.8]{Glockner06Libro} An n-ary semi-fuzzy quantifier $Q$ on a
base set $E\neq\varnothing$ is a mapping $Q:\mathcal{P}\left(  E\right)
^{n}\longrightarrow\mathbf{I}$.
\end{definition}

$Q$ assigns a gradual result to each pair of crisp sets $\left(  Y_{1}%
,\ldots,Y_{n}\right)  $. Some examples of semi-fuzzy quantifiers are:
\begin{align*} 
& \mathbf{about\_10}\left(  Y_{1},Y_{2}\right)  = T_{6,8,12,14}\left(\left\vert Y_{1}\cap Y_{2}\right\vert \right)\\
& \mathbf{about\_60\%\_or\_more}\left(  Y_{1},Y_{2}\right)  
 \qquad = \left\{
\begin{array}
[c]{cc}%
S_{0.4,0.6}\left(  \frac{\left\vert Y_{1}\cap Y_{2}\right\vert }{\left\vert
	Y_{1}\right\vert }\right)  & X_{1}\neq\varnothing\\
1 & X_{1}=\varnothing
\end{array}
\right. 
\end{align*}
where $T_{6,8,12,14}\left(  x\right)  $ and $S_{0.4,0.6}\left(  x\right)  $
represent the ordinary trapezoidal\footnote{Function $T_{a,b,c,d}$ is defined as
\[
T_{a,b,c,d}\left(  x\right)  =\left\{
\begin{array}
[c]{cc}%
0 & x\leq a\\
\frac{x-a}{b-a} & a<x\leq b\\
1 & b<x\leq c\\
1-\frac{x-c}{d-c} & c<x\leq d\\
0 & d<x
\end{array}
\right.
\]	 
} and $S$ fuzzy numbers\footnote{Function $S_{\alpha,\gamma}$ is defined as
\[
S_{\alpha,\gamma}\left(  x\right)  =\left\{
\begin{tabular}
[c]{ll}%
$0$ & $x<\alpha$\\
$2\left(  \frac{\left(  x-\alpha\right)  }{\left(  \gamma-\alpha\right)
}\right)  ^{2}$ & $\alpha<x\leq\frac{\alpha+\gamma}{2}$\\
$1-2\left(  \frac{\left(  x-\gamma\right)  }{\left(  \gamma-\alpha\right)
}\right)  ^{2}$ & $\frac{\alpha+\gamma}{2}<x\leq\gamma$\\
$1$ & $\gamma<x$%
\end{tabular}
\right.
\]	 
}.

We generally will denominate the fuzzy numbers used in the definition of
the semi-fuzzy quantifiers as `support functions of the semi-fuzzy quantifiers'.

Although the semantics of semi-fuzzy quantifiers is intuitive, they do not
permit to evaluate fuzzy quantified expressions. In \cite{Glockner06Libro},
the author proposes to use an additional mechanism to transform semi-fuzzy
quantifiers into fuzzy quantifiers. This mechanism allows to map semi-fuzzy
quantifiers into fuzzy quantifiers:

\begin{definition}
\cite[definition 2.10]{Glockner06Libro} A quantifier fuzzification mechanism
(\textit{QFM}) $\mathcal{F}$ assigns to each semi-fuzzy quantifier
$Q:\mathcal{P}\left(  E\right)  ^{n}\rightarrow\mathbf{I}$ a corresponding
fuzzy quantifier $\mathcal{F}\left(  Q\right)  :\widetilde{\mathcal{P}}\left(
E\right)  ^{n}\rightarrow\mathbf{I}$ of the same arity $n\in\mathbb{N}$ and on
the same base set $E$.
\end{definition}

\section{The \textit{QFM} $\mathcal{F}^{A}$\label{SubSectionModeloFA}}

In this section we present the finite \textit{QFM} $\mathcal{F}^{A}$
\cite{DiazHermida04IPMU,DiazHermida04-IEEE,DiazHermida06Tesis,DiazHermida10Arxiv,DiazHermida17-FuzzySets}%
. The $\mathcal{F}^{A}$ \textit{QFM} can be defined using two different
strategies. The first definition uses the equipotence concept \cite{Bandler80}
and remains purely on the use of fuzzy operators. The second is based on a
probabilistic interpretation of fuzzy sets. Both definitions are equivalent.

Following \cite{Bandler80}\ the equipotence between a crisp set $Y$ and a
fuzzy set $X$ can be defined as:
\begin{equation}
Eq\left(  Y,X\right)  \\
= \wedge_{e\in E}\left(  \mu_{X}\left(  e\right)
\rightarrow\mu_{Y}\left(  e\right)  \right)  \wedge\left(  \mu_{Y}\left(
e\right)  \rightarrow\mu_{X}\left(  e\right)  \right). \label{Eq_equipotence_1}
\end{equation}%
The concept of equipotence is basically a measure of equality between fuzzy sets.

Let us consider the product tnorm ($\wedge\left(  x_{1},x_{2}\right)
=x_{1}\cdot x_{2}$) and the Lukasiewicz implication ($\rightarrow\left(
x_{1},x_{2}\right)  =\min\left(  1,1-x_{1}+x_{2}\right)  $). In previous
expression, if $e\in Y$ then $\mu_{Y}\left(  e\right)  =1$ and if $e\notin Y$
then $\mu_{Y}\left(  e\right)  =0$. Then%


\begin{equation}
\left(  \mu_{X}\left(  e\right)  \rightarrow\mu_{Y}\left(  e\right)  \right)
\wedge\left(  \mu_{Y}\left(  e\right)  \rightarrow\mu_{X}\left(  e\right)
\right)  \\
=\left\{
\begin{array}
[c]{ccc}%
\mu_{X}\left(  e\right)  & : & e\in Y\\
1-\mu_{X}\left(  e\right)  & : & e\notin Y
\end{array}
\right.  
\label{Eq_equipotence_2}%
.\end{equation}

Then from (\ref{Eq_equipotence_1}) and (\ref{Eq_equipotence_2})


\begin{align*}
Eq\left(  Y,X\right)   &   ={\prod\limits_{e\in Y}}\mu_{X}\left(  e\right)  {\prod\limits_{e\in
		E\backslash Y}}\left(  1-\mu_{X}\left(  e\right)  \right).
\end{align*}

Using the equipotence concept, the $\mathcal{F}^{A}$ model can be defined as:

\begin{definition}
{ }Let $Q:\mathcal{P}\left(  E\right)  ^{n}\rightarrow\mathbf{I}$ be a
semi-fuzzy quantifier, $E$ finite. The \textit{QFM} $\mathcal{F}^{A}$ is
defined as:
\begin{equation*}
\mathcal{F}^{A}(X_{1}, \ldots, X_{n}) = \bigvee\limits_{Y_{1} \in \mathcal{P}(E)} \ldots \bigvee\limits_{Y_{n} \in \mathcal{P}(E)} \\
Eq(Y_1,X_1) \wedge \ldots \wedge Eq(Y_n,X_n) \wedge Q(Y_1, \ldots, Y_n)
\end{equation*}
where $\vee$ is the Lukasiewicz tconorm ($\vee\left(  x_{1},x_{2}\right)
=\min\left(  x_{1}+x_{2},1\right)  $), and $\wedge$ is the product tnorm
($\wedge\left(  x_{1},x_{2}\right)  =x_{1}\cdot x_{2}$).\smallskip{}
\end{definition}

Now, we will present an alternative definition based on a probabilistic
interpretation of fuzzy sets. The semantic interpretation of fuzzy sets based
on likelihood functions
\cite{Mabuchi92,Thomas95,Tursken2000Fundamentals,Dubois2000Fundamentals}
simply interprets vagueness in the data as a consequence of making a random
experiment in which a set of individuals are asked about the fulfillment of a
certain property.

For example, let us consider $h\in\mathbb{R}$. We can define the degree of
fulfillment of the statement \textit{\textquotedblleft the value of height
}$h$\textit{\ is tall\textquotedblright} as:
\begin{align*}
\mu\left(  \text{\textquotedblleft}h\text{ is }tall\text{\textquotedblright%
}\right)   &  =\Pr\left(  \text{\textquotedblleft}h\text{ is considered
}tall\text{\textquotedblright}\right) 
  =\frac{\left\vert v\in V:C\left(  v,\text{\textquotedblleft}h\text{ is
considered }tall\text{\textquotedblright}\right)  =1\right\vert }{\left\vert
V\right\vert }%
\end{align*}
where $V$ is a set of voters and $C\left(  v,\text{\textquotedblleft}h\text{
is considered }tall\text{\textquotedblright}\right)  $ denotes the answer of
the voter $v$ to the question \textquotedblleft$h$ is considered
$tall$\textquotedblright.

We can apply the same idea to compute the probability that a crisp set
$Y\in\mathcal{P}\left(  E\right)  $ is a representative of a fuzzy set
$X\in\widetilde{\mathcal{P}}\left(  E\right)  $ when we suppose the base set
$E$ finite and that the probabilities of the different elements are
independent. The intuition is to measure the probability that only the
elements in $Y$ belongs to $X$:

\begin{definition}
\label{DefInterpretProbConj}Let $X\in\widetilde{\mathcal{P}}\left(  E\right)
$ be a fuzzy set, $E$ finite. The probability of the crisp set $Y\in
\mathcal{P}\left(  E\right)  $ being a representative of the fuzzy set
$X\in\widetilde{\mathcal{P}}\left(  E\right)  $ is defined as
\begin{align*}
\Pr\left(  representative_{X} = Y\right) & = m_{X}\left(Y\right) 
 = {\prod\limits_{e\in Y}}\mu_{X}\left(  e\right)  {\prod\limits_{e\in E\backslash Y}%
}\left(  1-\mu_{X}\left(  e\right)  \right)
\end{align*}
\end{definition}

We would like to point out that in the previous definition the probability
points are the subsets of $E$. In this way the $\sigma$-algebra on which the
probability is defined is $\mathcal{P}\left(  E\right)  $.

Using expression \ref{DefInterpretProbConj} the definition of the \textit{QFM}
$\mathcal{F}^{A}$ is easily made:

\begin{definition}
\cite[page. 1359]{DiazHermida04IPMU}. Let $Q:\mathcal{P}\left(  E\right)
^{n}\rightarrow\mathbf{I}$ be a semi-fuzzy quantifier, $E$ finite. The
\textit{QFM} $\mathcal{F}^{A}$ is defined as
\begin{equation}
\mathcal{F}^{A}\left(  Q\right)  \left(  X_{1},\ldots,X_{n}\right)
= \sum_{Y_{1}\in\mathcal{P}\left(  E\right)  }\ldots\sum_{Y_{n}\in
	\mathcal{P}\left(  E\right)} \\
m_{X_{1}}\left(  Y_{1}\right)  \ldots m_{X_{n}}\left(  Y_{n}\right)  Q\left(  Y_{1},\ldots,Y_{n}\right)
\label{ModeloVerosimilitudes}%
\end{equation}
for all $X_{1},\ldots,X_{n}\in\widetilde{\mathcal{P}}\left(  E\right)  $.
\end{definition}

In expression \ref{ModeloVerosimilitudes} we are assuming that the probability
of being $Y_{i}$ a representative of the fuzzy set $X_{i}$ is independent of
the probability of being $Y_{j}$ a representative of the fuzzy set $X_{j}$ for
$i\neq j$. $\mathcal{F}^{A}\left(  Q\right)  \left(  X_{1},\ldots
,X_{n}\right)  $ can then be interpreted as the average opinion of
voters\footnote{We would like to point out that the probabilistic
interpretation of the \textit{QFM} $\mathcal{F}^{A}$ holds some relationships
with a similar probabilistic interpretation of the Zadeh's model. Let
$X\in\widetilde{\mathcal{P}}\left(  E\right)  $ be a fuzzy set representing
the linguistic concept \ \textit{`big houses'}, and let us suppose we want to
select an element $e\in E$. Let us assume we have the same probability of
selecting each element, and that $\mu_{e}\left(  X\right)  $ represents the
probability that the element $e$ fulfills the property of being a big house.
Let $fq:\left[  0,1\right]  \rightarrow\mathbf{I}$ be a function representing
a proportional unary linguistic quantifier (e.g. \textit{`most'}). Then, the
Zadeh's model just applies the linguistic quantifier to the average
probability of selecting an element fulfilling the property of being a `big
house':
\[
fq\left(  Avg\left(  \Pr\left(  e\_is\_big|e\_is\_selected\right)  \right)
\right)  =fq\left(  \frac{1}{m}\sum_{e\in E}\mu_{e}\left(  X\right)  \right)
\]
. In contrast, the $\mathcal{F}^{A}$ QFM computes the probability of every
possible combination in which the elements of $E$ can fulfill the property of
`being a big house'. After computing the probability of each combination, we
compute the average of applying the support function of the quantifier $fq$ to
the possible combinations.}.

The following example shows the application of the \textit{QFM} $\mathcal{F}%
^{A}$:

\begin{example}
\label{EjemVerosimilitudes}Let us consider the sentence:
\[
\text{\textquotedblleft Nearly all big houses are expensive\textquotedblright}%
\]
where the semi-fuzzy quantifier $Q=$\textbf{ }\textit{`nearly all'}, and the
fuzzy sets \textit{`big houses'} and \textit{`expensive'} take the following
values:
\begin{align*}
\mathbf{big}\text{ }\mathbf{houses}  &  =\left\{  0.8/e_{1},0.9/e_{2}%
,1/e_{3},0.2/e_{4}\right\}\\%
\mathbf{expensive}  &  =\left\{  1/e_{1},0.8/e_{2},0.3/e_{3},0.1/e_{4}\right\}
\\
Q\left(  X_{1},X_{2}\right)   &  =\left\{
\begin{array}
[c]{cc}%
\max\left\{  2\left(  \frac{\left\vert X_{1}\cap X_{2}\right\vert }{\left\vert
X_{1}\right\vert }\right)  -1,0\right\}  & X_{1}\neq\varnothing\\
1 & X_{1}=\varnothing
\end{array}
.\right.
\end{align*}
We compute the probabilities of the representatives of the fuzzy sets
\textit{`big houses'} and \textit{`expensive'}:
\begin{align*}
& m_{\mathbf{big}\text{ }\mathbf{houses}}\left(  \varnothing\right) 
=\left(  1-0.8\right)  \left(  1-0.9\right)  \left(  1-1\right)  \left(1-0.2\right)  =0, \\
& \ldots \\
& m_{\mathbf{expensive}}\left(  \left\{  e_{1},e_{2},e_{3},e_{4}\right\} \right) = 0.8\cdot0.9\cdot1\cdot0.2=0.144, \\
& m_{\mathbf{expensive}}\left(  \varnothing\right) =\left(  1-1\right) \left(  1-0.8\right)  \left(  1-0.3\right)  \left(  1-0.1\right)  =0,\\
& m_{\mathbf{expensive}}\left(  \left\{  e_{1}\right\}  \right) = 1\cdot\left(  1-0.8\right)  \left(  1-0.3\right)  \left(  1-0.1\right) =0.126,\\
& \ldots \\
& m_{\mathbf{expensive}}\left(  \left\{  e_{1},e_{2},e_{3},e_{4}\right\} \right) =0.8\cdot0.9\cdot1\cdot0.2=0.144.
\end{align*}
And using expression (\ref{ModeloVerosimilitudes}):
\begin{align*}
&  \mathcal{F}^{A}\left(  Q\right)  \left(  \mathbf{big}\text{ }%
\mathbf{houses},\mathbf{expensive}\right) 
  =\sum_{Y_{1}\in\mathcal{P}\left(  E\right)  }\sum_{Y_{2}\in\mathcal{P}%
\left(  E\right)  }m_{X_{1}}\left(  Y_{1}\right)  m_{X_{2}}\left(
Y_{2}\right)  Q\left(  Y_{1},Y_{2}\right)  =0.346.
\end{align*}

\end{example}

\section{The DFS axiomatic framework}

We will present now the definition of the \textit{Determiner fuzzification
scheme (DFS}) axiomatic framework \cite{Glockner06Libro}. It is impossible in
this paper to explain in full detail the DFS axiomatic framework, as in
\cite{Glockner06Libro} the author needed chapters three and four to present it
in adequate detail. We will limit us to introduce the framework, referring the
reader to the previous reference for further study. 

\begin{definition}
A \textit{QFM} $\mathcal{F}$ is called a determiner fuzzification scheme (DFS)
if the conditions listed in $TABLE$ $I$ are satisfied for all semi-fuzzy quantifiers $Q:\mathcal{P}\left(  E\right)  ^{n}\longrightarrow\mathbf{I}$.
\end{definition}

\begin{table*}[!t]
\begin{center}
\caption{Conditions of a DFS for all semi-fuzzy quantifiers $Q:\mathcal{P}\left(  E\right)  ^{n}\longrightarrow\mathbf{I}$\label{tab:properties}}
\begin{tabular}[c]{l p{6cm} c}
\toprule	
\textbf{Name} & \textbf{Condition} & \textbf{Reference}  \\
\toprule
Correct generalization & $\mathcal{U}\left(  \mathcal{F}\left(  Q\right)
\right)  =Q$\quad if $n\leq1$ & (Z-1) \\
\midrule
Projection quantifiers & $\mathcal{F}\left(  Q\right)  =\widetilde{\pi_{e}}%
$\quad if $Q=\pi_{e}$ for some $e\in E$ & (Z-2)\\
\midrule
Dualisation & $\mathcal{F}\left(  Q\square\right)  =\mathcal{F}%
\left(  Q\right)  \widetilde{\square}$\quad$n>0$ & (Z-3)\\
\midrule
Internal joins & $\mathcal{F}\left(  Q\cup\right)  =\mathcal{F}\left(
Q\right)  \widetilde{\cup}$\quad$n>0$ & (Z-4)\\
\midrule
Preservation of monotonicity & If $Q$ is nonincreasing in the $n$-th arg,
then $\mathcal{F}\left(  Q\right)  $ is nonincreasing in $n$-th arg, $n>0$  & (Z-5)\\
\midrule
Functional application & $\mathcal{F}\left(  Q\circ\underset{i=1}{\overset
{n}{\times}}\widehat{f_{i}}\right)  =\mathcal{F}\left(  Q\right)
\circ\underset{i=1}{\overset{n}{\times}}\widehat{\mathcal{F}}\left(
f_{i}\right)  $ where $f_{1},\ldots,f_{n}:E^{\prime}\rightarrow E,E^{\prime}\neq\varnothing$ & (Z-6) \\
\bottomrule
\end{tabular}
\end{center}
\end{table*}

In the following section,
we will present the main properties of the models fulfilling the framework in
relation to the $\mathcal{F}^{A}$ QFM, and some others that are not consequence of the DFS axiomatic framework but which are important in order to adequately characterize the behavior of QFMs.

\section{Analysis of the behavior of the $\mathcal{F}^{A}$ QFM}

In this section we will give a general overview of the main properties of the
$\mathcal{F}^{A}$ QFM referring the different publications where the proofs
and extended explanations can be found.

\subsection{Main properties of the $\mathcal{F}^{A}$ QFM derived from the DFS
framework\label{PropertiesDerivedDFS}}

As we have advanced, the $\mathcal{F}^{A}$ QFM fulfills the DFS axiomatic
framework. This fact guarantees that it also fulfills all the adequacy
properties the framework guarantees. We will present now the main properties derived from it in relation with the behavior of the $\mathcal{F}^{A}$ QFM.

\subsubsection{Correct generalization (P1)}

This property is possibly the most important property derived from the DFS
framework. \textit{Correct generalization} requires that the behavior of a
fuzzy quantifier $\mathcal{F}\left(  Q\right)  $ when we apply it to crisp
arguments would be equal to the application of the semi-fuzzy quantifier $Q$
over the same crisp arguments. That is, for all the crisp subsets
$Y_{1},\ldots,Y_{n}\in\mathcal{P}\left(  E\right)  $, then it holds that
$\mathcal{F}\left(  Q\right)  \left(  Y_{1},\ldots,Y_{n}\right)  =Q\left(
Y_{1},\ldots,Y_{n}\right)  $. For example, given the crisp sets $\mathbf{big}$
$\mathbf{houses},\mathbf{expensive}\in\mathcal{P}\left(  E\right)  $, this
property guarantees that:
\begin{equation*}
\mathcal{F}\left(\mathbf{some}\right)  \left(  \mathbf{big}\text{ }\mathbf{houses},\mathbf{expensive}\right) = \\
\mathbf{some}\left(\mathbf{big}\text{ }\mathbf{houses},\mathbf{expensive}\right)
\end{equation*}

The proof of this property for the $\mathcal{F}^{A}$ QFM can be found in
\cite[page 291]{DiazHermida06Tesis},\cite[page 31]{DiazHermida10Arxiv} for the
unary case, as in conjunction with the other axioms of the DFSs is enough to
assure the fulfillment of the property in the general case.

\subsubsection{Quantitativity (P2)}

In TGQ, a quantifier is \textit{quantitative }if it does not depend on any
particular property fulfilled by the elements. Most common examples of
quantifiers we can find in the literature are quantitative (e.g.,
\textit{`many', `about 10'}). \textit{Non-quantitative quantifiers} involve
the reference to particular elements of the base set (e.g., \textit{`Spain'}
in a set of countries). A \textit{QFM} $\mathcal{F}$ retains the
\textit{quantitativity property} if quantitative semi-fuzzy quantifiers are
converted into quantitative fuzzy quantifiers by the application of
$\mathcal{F}.$ The fulfillment of the quantitativity property for the
$\mathcal{F}^{A}$ QFM is a consequence of the fulfillment of the DFS framework.

\subsubsection{Projection quantifier (P3)}

The \textit{Axiom Z-2} of the DFS framework establishes that the
\textit{projection crisp quantifier} $\pi_{e}\left(  Y\right)  $ (which
returns $1$ if $e\in Y$ and $0$ in other case) is transformed into the
\textit{fuzzy projection quantifier} $\widetilde{\pi_{e}}\left(  X\right)  $
(which returns $\mu_{X}\left(  e\right)  $). The proof of this property for
the $\mathcal{F}^{A}$ QFM can be found in \cite[page 272]{DiazHermida06Tesis},
\cite[page 31]{DiazHermida10Arxiv}.

\subsubsection{Induced propositional logic (P4)}

In \cite{Glockner06Libro} a mechanism was proposed to embed crisp logical
functions ($\lnot\left(  x\right)  $, $\wedge\left(  x_{1},x_{2}\right)  $,
$\vee\left(  x_{1},x_{2}\right)  $, $\rightarrow\left(  x_{1},x_{2}\right)  $)
into semi-fuzzy quantifiers. For example, the `and' function can be embedded
into a semi-fuzzy quantifier $Q_{\wedge}:\mathcal{P}\left(  \left\{
e_{1},e_{2}\right\}  \right)  \rightarrow\left\{  0,1\right\}  $ such that
$Q_{\wedge}\left(  \varnothing\right)  =Q_{\wedge}\left(  \left\{
e_{1}\right\}  \right)  =Q_{\wedge}\left(  \left\{  e_{2}\right\}  \right)
=0$ and $Q_{\wedge}\left(  \left\{  e_{1},e_{2}\right\}  \right)  =1$. The
property of \textit{induced propositional logic} assures that crisp logical
functions are transformed into acceptable fuzzy logical functions. In the case
of the $\mathcal{F}^{A}$ QFM the induced propositional functions are
respectively the \textit{strong negation}, the probabilistic \textit{tnorm},
the probabilistic \textit{tconorm }and the \textit{Rechenbach}\emph{ }fuzzy
implication. The proof of this property for the $\mathcal{F}^{A}$ QFM can be
found in \cite[page 265]{DiazHermida06Tesis}, \cite[page 28]%
{DiazHermida10Arxiv}.

\subsubsection{External negation (P5)}

We will say that a QFM fulfills the \textit{external negation property} if
$\mathcal{F}\left(  \widetilde{\lnot}Q\right)  $ is equivalent to
$\widetilde{\lnot}\mathcal{F}\left(  Q\right)  $. In words\textit{,}
equivalence of expressions \textit{\textquotedblleft it is false that at least
60\% of the good students are good athletes\textquotedblright\ }and
\textit{\textquotedblleft less than 40\% of the good students are good
athletes\textquotedblright\ }is assured. Here, $\widetilde{\lnot}$ is assumed
to be the induced negation of the QFM (the strong negation for the
$\mathcal{F}^{A}$ QFM). The proof of this property for the $\mathcal{F}^{A}$
QFM can be found in \cite[page 273]{DiazHermida06Tesis}, \cite[page
32]{DiazHermida10Arxiv}.

\subsubsection{Internal negation (P6)}

The \textit{internal negation or antonym} of a semi-fuzzy quantifier
$Q:\mathcal{P}\left(  E\right)  ^{n}\longrightarrow\mathbf{I}$ is defined as
$Q\mathbf{\lnot}\left(  Y_{1},\ldots,Y_{n}\right)  =Q\left(  Y_{1}%
,\ldots,\mathbf{\lnot}Y_{n}\right)  $. For example, \textit{`all'} is the
antonym of \textit{`no'} as $\mathbf{no}\left(  Y_{1},Y_{2}\right)
=\mathbf{all}\left(  Y_{1},\lnot Y_{2}\right)  =\mathbf{all}\lnot\left(
Y_{1},Y_{2}\right)  $. Fulfillment of the \textit{internal negation property}
assures that internal negation transformations are translated to the fuzzy
case. The proof of this property for the $\mathcal{F}^{A}$ QFM can be found in
\cite[page 273]{DiazHermida06Tesis}, \cite[page 32]{DiazHermida10Arxiv}.

\subsubsection{Dualisation (P7)}

The \textit{dualisation property }is a consequence of the fulfillment of the
external negation and internal negation properties. In conjunction, these
negation properties assure the maintenance of the equivalences in the
`Aristotelian square' \cite{Gamut84}. It forms part of the DFS framework
(\textit{Z-3 axiom}), being the dual of a semi-fuzzy quantifier $Q:\mathcal{P}%
\left(  E\right)  ^{n}\longrightarrow\mathbf{I}$ defined as $Q\widetilde
{\square}\left(  Y_{1},\ldots,Y_{n}\right)  =\widetilde{\lnot}Q\left(
Y_{1},\ldots,\mathbf{\lnot}Y_{n}\right)  $ and equivalently in the fuzzy case.
As an example, the equivalence of $\mathcal{F}\left(  \mathbf{no}\right)
\left(  \mathbf{big}\text{ }\mathbf{houses},\widetilde{\lnot}%
\mathbf{expensive}\right)  $ and $\mathcal{F}\left(  \mathbf{all}\right)
\left(  \mathbf{big}\text{ }\mathbf{houses},\mathbf{expensive}\right)  $ is
assured; or in words, \textit{\textquotedblleft no big house is not
expensive\textquotedblright{}} and \textit{\textquotedblleft all big houses are
expensive\textquotedblright} are equivalent.\textit{ }The proof of this
property for the $\mathcal{F}^{A}$ QFM can be found in \cite[page
275]{DiazHermida06Tesis}, \cite[page 33]{DiazHermida10Arxiv}.

\subsubsection{Union/intersection of arguments P8}

The properties of \textit{union and intersection of arguments} guarantee the
compliance with some transformations to construct new quantifiers using unions
and intersections of arguments. Being $Q:\mathcal{P}\left(  E\right)
^{n+1}\longrightarrow\mathbf{I}$ an $n+1$-ary semi-fuzzy quantifier, $Q\cup$
is defined as $Q\cup\left(  Y_{1},\ldots,Y_{n+1}\right)  =Q\left(
Y_{1},\ldots,Y_{n}\cup Y_{n+1}\right)  $ and equivalently in the fuzzy case.
\textit{Z-4 axiom} specifies this property for the union of quantifiers, as
the property is also fulfilled for the intersection of arguments as a
consequence of the DFS axiomatic framework.

One particular example of the consequences of fulfilling these properties is
that the equivalence between absolute unary and binary quantifiers is assured,
guaranteeing that we obtain the same result when we evaluate
\textit{\textquotedblleft around 5 big houses are expensive\textquotedblright}
and \textit{\textquotedblleft there are around 5 houses that are big and
expensive\textquotedblright, }where the evaluation of the first quantified
expression is computed by means of an absolute binary quantifier and the
evaluation of the second expression is computed by applying the corresponding
absolute unary quantifier to the intersection of `big houses' and `expensive
houses' computed by means of the induced \textit{tnorm. }In combination with
the internal and external negation properties, they allow the preservation of
the boolean argument structure that can be expressed in natural language when
none of the boolean variables $X_{i}$ occurs more than
once\ \cite[section 3.6]{Glockner06Libro}. The proof of these properties for
the $\mathcal{F}^{A}$ QFM can be found in \cite[page 275]{DiazHermida06Tesis},
\cite[page 33]{DiazHermida10Arxiv}.\textit{\emph{ }}

\subsubsection{Coherence with standard quantifiers P9}

By \textit{standard quantifiers} we mean the classical quantifiers
$\exists,\forall$ and their binary versions $\mathbf{some}$ and $\mathbf{all}%
$. Every QFM fulfilling the DFS axiomatic framework guarantees that the fuzzy
version of these classical quantifiers is the expected. For example, the
$\mathcal{F}^{A}$ QFM fulfills (where $\widetilde{\vee},\widetilde{\wedge
},\widetilde{\rightarrow}$ are the logical operators induced by the
$\mathcal{F}^{A}$\textit{ }model):
\begin{equation*}
\mathcal{F}\left(  \exists\right)  \left(  X\right) =\sup\left\{\overset{m}{\underset{i=1}{\widetilde{\vee}}}\mu_{X}\left(  a_{i}\right): \right.\\
\left. {} A=\left\{  a_{1},\ldots,a_{m}\right\}  \in\mathcal{P}\left(  E\right),a_{i}\neq a_{j}\text{ if }i\neq j\right. \bigg\}
\end{equation*}
\begin{equation*}
\mathcal{F}\left(  \mathbf{all}\right)  \left(  X_{1},X_{2}\right) = \inf\left\{  \overset{m}{\underset{i=1}{\widetilde{\wedge}}}\mu_{X_{1}} \left(  a_{i}\right)  \widetilde{\rightarrow}\mu_{X_{2}}\left(  a_{i}\right) : \right.\\
 \left. {} A=\left\{a_{1},\ldots,a_{m}\right\}  \in\mathcal{P}\left(  E\right), a_{i}\neq a_{j}\text{ if }i\neq j\right. \bigg\}
\end{equation*}

This property is a consequence of being the $\mathcal{F}^{A}$ QFM a DFS.

\subsubsection{Monotonocity in arguments P10}

In \cite{Glockner06Libro} different definitions to assure the preservation of
monotonicity relationships were included. The property of \textit{monotonicity
in arguments}, which forms part of the DFS axiomatic framework (\textit{axiom
Z5}) assures that monotonic behaviors in arguments are translated from the
semi-fuzzy to the fuzzy case. For example, for the binary semi-fuzzy
quantifier \textit{`most',} that is increasing in its second argument (e.g.
\textit{\textquotedblleft most politics are rich\textquotedblright}), the
fulfillment of this property guarantees that its fuzzy version will also be
increasing in its second argument. The DFS framework also guarantees the
maintenance of `local' monotonicity properties \cite[section 4.11]%
{Glockner06Libro}. The proof of these properties for the $\mathcal{F}^{A}$ QFM
can be found in \cite[page 282]{DiazHermida06Tesis}, \cite[page 39]%
{DiazHermida10Arxiv}.

\subsubsection{Monotonicity between quantifiers P11}

The DFS axiomatic framework also guarantees the preservation of
\textit{monotonicity relationships between quantifiers}\emph{.} For example,
\textit{`between 4 and 6'} is more specific than \textit{`between 2 and 8'}.
Thanks to this property, monotonicity relationships between semi-fuzzy
quantifiers are preserved between fuzzy quantifiers. The fulfillment of the
property of monotonicity in quantifiers is a consequence of the DFS axiomatic framework.

\subsubsection{Crisp argument insertion P12}

The operator of \textit{crisp argument insertion,} applied to a semi-fuzzy
quantifier $Q:\mathcal{P}\left(  E\right)  ^{n}\rightarrow\mathbf{I}$, allows
to construct a new quantifier $Q:\mathcal{P}\left(  E\right)  ^{n-1}%
\rightarrow\mathbf{I}$ by means of the restriction of $Q$ by a crisp set $A$.
More explicitly, the crisp argument insertion $Q\lhd A$ is defined as $Q\lhd
A\left(  Y_{1},\ldots,Y_{n-1}\right)  =Q\left(  Y_{1},\ldots,Y_{n-1},A\right)
$. A QFM preserves this property if $\mathcal{F}\left(  Q\lhd A\right)
=\mathcal{F}\left(  Q\right)  \lhd A$; that is, the crisp argument insertion
commutes for semi-fuzzy and fuzzy quantifiers. Crisp argument insertion allows
to model the `adjectival restriction' of natural language in the crisp case.
The fulfillment of this property by the $\mathcal{F}^{A}$ QFM is also a
consequence of the DFS axiomatic framework.

\subsection{Some relevant properties considered in the QFM framework but not
derived from the DFS axioms \label{PropertiesAdditionalDFS}}

In \cite[chapter six]{Glockner06Libro} it can be found the definition of some
additional adequacy properties for characterizing QFMs. These properties were
not included in the DFS framework in some cases, for not being compatible with
it, and in other cases, in order to not excessively constraint the set of
theoretical models fulfilling the DFS framework. We will present now the most
relevant ones:

\subsubsection{Continuity in arguments P13}

The property of \textit{continuity in arguments}\emph{ }assures the continuity
of the models with respect to the input sets. It is fundamental to guarantee
that small variations in the inputs do not cause jumps in the outputs.

The $\mathcal{F}^{A}$ QFM is a finite DFS and it is continuous. The proof of
this property for the $\mathcal{F}^{A}$ QFM can be found in \cite[page
293]{DiazHermida06Tesis}, \cite[page 48]{DiazHermida10Arxiv}.

\subsubsection{Continuity in quantifiers P14}

The property of \textit{continuity in quantifiers}\emph{ }assures the
continuity of the QFMs with respect to small variations in the quantifiers.
The proof of this property for the $\mathcal{F}^{A}$ QFM can be found in
\cite[page 297]{DiazHermida06Tesis}, \cite[page 48]{DiazHermida10Arxiv}.

\subsubsection{Propagation of fuzziness P15}

\textit{Propagation of fuzziness properties }assure that fuzzier inputs
(understood as fuzzier input sets) and fuzzier quantifiers produce fuzzier
outputs\footnote{Let be $\preceq_{c}$ a partial order in $\mathbf{I}%
\times\mathbf{I}$ defined as \cite[section 5.2 and 6.3]{Glockner06Libro}:%
\[
x\preceq_{c}y\Leftrightarrow y\leq x\leq\frac{1}{2}\text{ or }\frac{1}{2}\leq
x\leq y
\]
for $x,y\in\mathbf{I}$. A fuzzy set $X_{1}$ is at least as fuzzy as a fuzzy
set $X_{2}$ if for each $e\in E$, $\mu_{X_{1}}\left(  e\right)  \preceq_{c}%
\mu_{X_{2}}$; that is, membership degrees of $X_{1}$ are closer to $0.5$ than
membership degrees of $X_{2}$. In the case of fuzzy quantifiers a similar
definition is applied.}. This property is not fulfilled by the $\mathcal{F}%
^{A}$ QFM because it is not fulfilled by the induced product \textit{tnorm}
and the induced probabilistic sum \textit{tconorm} of the model. An extensive
analysis of the fulfillment of this property by the main QFMs that can be
found in the literature is presented in \cite{DiazHermida17-FuzzySets}.

\subsubsection{Fuzzy argument insertion P16}

The property of \emph{fuzzy argument insertion }is the fuzzy counterpart of
the crisp argument insertion. To our knowledge, this property has only been
proved for the DFSs $\mathcal{M}_{CX}$ \cite[definition 7.56]{Glockner06Libro}
and $\mathcal{F}^{A}$ \cite[page 292]{DiazHermida06Tesis},\cite[page
48]{DiazHermida10Arxiv}.

\subsection{Additional properties fulfilled by the $\mathcal{F}^{A}$ QFM not
included in the QFM framework}

In this section we summarize three other properties fulfilled by the
$\mathcal{F}^{A}$ QFM that do not form part of the ones considered in the QFM
framework by Gl\"{o}ckner \cite{Glockner06Libro}. We will explain these
properties in some more detail as they are not commonly considered in the
bibliography about fuzzy quantification. In \cite{DiazHermida17-FuzzySets}
these properties were used, in combination with other criteria, to present a
comparison of the behavior of different QFMs thinking in their convenience
for practical applications.

\subsubsection{Property of averaging for the identity
quantifier\label{SubSubSubPropMedia}}

The fulfillment of this property by a QFM $\mathcal{F}$ assures that when we
apply the model to the unary semi-fuzzy quantifier $\mathbf{identity}\left(
Y\right)  =\frac{\left\vert Y\right\vert }{\left\vert E\right\vert }%
,Y\in\mathcal{P}\left(  E\right)  $ we obtain the average of the membership
grades. For the `identity' semi-fuzzy quantifier the addition of one element
increases the result in $\frac{1}{m}$. We could expect that a QFM
$\mathcal{F}$ would translate this linearity relationship into the fuzzy case.

The QFM $\mathcal{F}^{A}$ fulfills the property of averaging for the identity
quantifier that assures:
\[
\mathcal{F}^{A}\left(  \mathbf{identity}\right)  \left(  X\right)  =\frac
{1}{m}\sum_{j=1}^{m}\mu_{X}\left(  e_{j}\right)
\]
The proofs can be found in \cite[page 298]{DiazHermida06Tesis} or in
\cite[page 50]{DiazHermida10Arxiv}.

\subsubsection{Property of the probabilistic interpretation of
quantifiers\label{SubSubSubPropRecubProbab}}

Let us suppose we use a set of semi-fuzzy quantifiers
(\textit{\textquotedblleft at most about 20\%\textquotedblright},
\textit{\textquotedblleft between 20\% and 80\%\textquotedblright},
\textit{\textquotedblleft at least about 80\%\textquotedblright}) to split the
quantification universe. We will say that a set of semi-fuzzy quantifiers
$Q_{1},\ldots,Q_{r}:\mathcal{P}^{n}\left(  E\right)  \rightarrow\mathbf{I}$
forms a \textit{quantified Ruspini partition} of the quantification universe
if for all $Y_{1},\ldots,Y_{n}\in\mathcal{P}\left(  E\right)  $ it holds that
\[
Q_{1}\left(  Y_{1},\ldots,Y_{n}\right)  +\ldots+Q_{r}\left(  Y_{1}%
,\ldots,Y_{n}\right)  =1
\]
The QFM $\mathcal{F}^{A}$ translates this relationship to the fuzzy case.
Forming $Q_{1},\ldots,Q_{r}:\mathcal{P}\left(  E\right)  ^{n}\rightarrow
\mathbf{I}$ a quantified Ruspini partition it is fulfilled:%
\[
\mathcal{F}^{A}\left(  Q_{1}\right)  \left(  X_{1},\ldots,X_{n}\right)
+\ldots+\mathcal{F}^{A}\left(  Q_{r}\right)  \left(  X_{1},\ldots
,X_{n}\right)  =1
\]

This property is very interesting because it will permit to interpret the
result of evaluating a fuzzy quantified expression as a probability
distributed over the labels related to the quantifiers. Proofs can be found in
\cite[page 298]{DiazHermida06Tesis} or in \cite[page 52]{DiazHermida10Arxiv}.

\subsubsection{Fine distinction between
objects\label{SectionRankingGeneration}}

This property is particularly useful for the application of fuzzy quanfiers in
ranking problems. Let us consider a set of objects $o_{1},\ldots,o_{N}$ for
which the fulfillment of a set of criteria $p_{1},\ldots,p_{m}$ is represented
by means of a fuzzy set $X^{o_{i}}=\left\{  \mu_{X^{i}}\left(  p_{1}\right)
/p_{1},\ldots,\mu_{X^{i}}\left(  p_{m}\right)  /p_{m}\right\}  $, where
$\mu_{X^{i}}\left(  p_{j}\right)  /p_{j}$ indicates the fulfillment of the
criteria $p_{j}$ by the object $o_{i}$. Generally, we also have a set of
weights $W=\left\{  \mu_{W}\left(  p_{1}\right)  /p_{1},\ldots,\mu_{W}\left(
p_{m}\right)  /p_{m}\right\}  $ to indicate the relative relevance of the
different criteria $p_{1},\ldots,p_{m}$.

Using fuzzy quantification, a ranking can be constructed assigning to each
object a weight computed by means of an unary proportional quantified
expression $r^{o_{i}}=\widetilde{Q}\left(  X^{o_{i}}\right)  $ (in the case
that a vector of weights is not involved) or a binary proportional quantified
expression $r^{o_{i}}=\widetilde{Q}\left(  W,X^{o_{i}}\right)  $ (in the case
that a vector of weights $W$ is used to indicate the relative importance of
each criteria). In this way, computing $r^{o_{i}}$ for each $i=1,\ldots,N$, we
can sort the objects of the collection with respect to the linguistic
expression \textit{`how }$\widetilde{Q}$' criteria are fulfilled (e.g., for
$\widetilde{Q}=\mathbf{many}$, \textit{`how many'}).

In order to guarantee a sufficient discriminative power, even small variations
in the inputs should produce some effect in the outputs. In \cite[section
5.6]{DiazHermida17-FuzzySets} it was proposed to analyze the behavior of QFMs
with respect to the following semi-fuzzy quantifiers defined by means of
increasing fuzzy numbers:

\begin{definition}
Let $h\left(  x\right)  :\left[  0,1\right]  \rightarrow\mathbf{I}$ be an
strictly increasing continuous mapping; i.e., $h\left(  x\right)  >h\left(
y\right)  $ for every $x>y$.\ We define the unary and binary semi-fuzzy
quantifiers $Q_{h}:\mathcal{P}\left(  E\right)  \rightarrow\mathbf{I}$ and
$Q_{h}:\mathcal{P}\left(  E\right)  ^{2}\rightarrow\mathbf{I}$ as
\begin{align*}
Q_{h}\left(  Y\right)   &  =h\left(  \left\vert Y\right\vert \right)
,Y\in\mathcal{P}\left(  E\right) \\
Q_{h}\left(  Y_{1},Y_{2}\right)   &  =\left\{
\begin{array}
[c]{cc}%
h\left(  \frac{\left\vert Y_{1}\cap Y_{2}\right\vert }{\left\vert
Y_{1}\right\vert }\right)  & Y_{1}\neq\varnothing\\
1 & Y_{1}=\varnothing
\end{array}
\right.
\end{align*}

\end{definition}

And then, to require to a QFM $\mathcal{F}$ the maintanace of the strictly
increasing relationships in the fuzzy case. That is, that any increase in the
fulfillment of a criteria will increase $\mathcal{F}\left(  Q_{h}\right)  $ in
the unary case, and that any increase in the fulfillment of a criteria
associated with a strictly positive weight will increase $\mathcal{F}\left(
Q_{h}\right)  $ in the binary case.

The $\mathcal{F}^{A}$ DFS fulfills this property as can be found in
\cite[section 5.6]{DiazHermida17-FuzzySets}.

\section{Limit case approximation of the $\mathcal{F}^{A}$ QFM}

In this section we will prove that in the general case of semi-fuzzy quantifiers defined by means of continuous proportional fuzzy numbers (i.e., `unary proportional', `binary proportional', `comparative proportional', etc.) the $\mathcal{F}^{A}$ QFM can be approximated by simply evaluating the fuzzy number that supports the quantifier over a function which depends on the average of the different boolean combinations of the input sets (more details below). As an additional result, in the specific case of unary and binary proportional linguistic quantifiers, the $\mathcal{F}^{A}$ QFM converges to the Zadeh's model when the intersection of the inputs sets is computed with the \textit{probabilistic tconorm} for binary proportional quantifiers.

Before proceeding, we will make a brief summary of the ideas of the proof in order to facilitate its understanding. In the proof, we will start introducing some previous results which guarantee that quantitative quantifiers can be expressed by means of a function of the cardinalities of the boolean combinations of the input sets. This will allow us to develop a general proof, valid for each quantitative quantifier defined by means of a proportional fuzzy number.

After that, we will use the fact that in the definition of the $\mathcal{F}^{A}$ QFM we are interpreting membership degrees $\mu_{X}(e_{i})$ as probabilities, and that independence is fulfilled for $\mu_{X}(e_{i}),\mu_{X}(e_{j}), i \neq j$. In this case, a fuzzy set  $X=\{a_1/e_1, \dots,a_{m}/e_{m}\}$ will induce an specific probability distribution over the function of the possible cardinalities $0, \dots,m$ of the set. In other words, as we are interested in the number of elements of $X$ fulfilling the property, each possible cardinality $i$ will have a probability value measuring the probability that exactly `$i$' elements fulfill the property. We will see that this probability follows a poisson binomial distribution. Moreover, we will also prove that the projections of the probability function $f(i_1,\dots,i_{K})$ induced by the $\mathcal{F}^{A}$ QFM for n-ary quantifiers follow poisson binomial distributions. In that case, the probability parameters of the $j$ projection will be determined by the $j$-th boolean combination used in the specification of the semi-fuzzy quantifier.

When $m$ tends to infinite, the fuzzy set $X=\{a_1/e_1, \dots, a_{m}/e_{m}\}$ will induce a sequence $B_1,B_2,\dots$ of poisson binomial distributions on $0, \dots, m$. But we will see that the variance of $Z_{i}=B_{i}/m$ will tend to 0. As in the definition of proportional quantifiers we use fuzzy numbers defined over $[0,1]$ instead of $\{0,\dots,m\}$, when we normalize the probability distribution $f(i_1,\dots,i_{K})$ to $[0,1]^n$ we will obtain a probability distribution whose projections are poisson binomial distributions such that their average converge in probability to the average of the membership degrees of the fuzzy set `induced' by the boolean combination, and their variance tend to $0$. Then, as each marginal distribution converges in probability to a constant, by the theorem of the continuous mapping the joint distribution converges in probability to a constant. In practice, this implies that the probability distribution will be more and more concentrated around the average of the boolean combinations as the size of the input sets tends to infinite. As a consequence, we could simply evaluate $\mathcal{F}^{A}$ computing the value of the proportonal fuzzy number used in the definition of $Q$ over the average of the boolean combinations.

After this summary we will present the proof in full detail.


The next theorem establishes that, in the
finite case, quantitative semi-fuzzy quantifiers can be expressed by means of
a function of the cardinalities of the boolean combinations of the input sets.

\begin{theorem}
\label{TeoremaCuantitativo}\cite[Theorem 11.32, chapter 11]{Glockner06Libro} A
semi-fuzzy quantifier $Q:\mathcal{P}\left(  E\right)  ^{n}\longrightarrow
\mathbf{I}$ on a finite base set $E\neq\varnothing$ is quantitative if and
only if $Q$ can be computed from the cardinalities of its arguments and their
Boolean combinations, i.e. there exist Boolean expressions $\Phi_{1}\left(
Y_{1},\ldots,Y_{n}\right)  ,\ldots,\Phi_{K}\left(  Y_{1},\ldots,Y_{n}\right)
$ for some $K\in\mathbb{N}$, and a mapping $q:\left\{  0,\ldots,m\right\}
^{K}\longrightarrow\mathbf{I}$ such that%
\begin{equation}
Q\left(  Y_{1},\ldots,Y_{n}\right)  = \\
q\left(  \left\vert \Phi_{1}\left(
Y_{1},\ldots,Y_{n}\right)  \right\vert ,\ldots,\left\vert \Phi_{K}\left(
Y_{1},\ldots,Y_{n}\right)  \right\vert \right)  \label{EqTeoremaCuantitativo}%
\end{equation}
for all $Y_{1},\ldots,Y_{n}\in\mathcal{P}\left(  E\right)  $.
\end{theorem}

We will also introduce the following notation for denoting the boolean combinations:

Let be $l_{1},\ldots,l_{n}\in\left\{  0,1\right\}  $, we define $\Phi
_{l_{1},\ldots,l_{n}}\left(  Y_{1},\ldots,Y_{n}\right)  $ as:%
\[
\Phi_{l_{1},\ldots,l_{n}}\left(  Y_{1},\ldots,Y_{n}\right)  =Y_{1}^{\left(
l_{1}\right)  }\cap\ldots\cap Y_{n}^{\left(  l_{n}\right)  }%
\]
where%
\[
Y^{\left(  l\right)  }=\left\{
\begin{tabular}
[c]{lll}%
$Y$ & $:$ & $l=1$\\
$\lnot Y$ & $:$ & $l=0$%
\end{tabular}
\ \right.
\]

Let us remember we are denoting $\left\vert E\right\vert =m$. Then, in the
finite case, quantitative semi-fuzzy quantifiers can be expressed by means of
a function $q:\left\{  0,\ldots,m\right\}  ^{K}\longrightarrow\mathbf{I}$
depending only of the cardinalities of the boolean combinations of
$Y_{1},\ldots,Y_{n}$. For example, proportional binary semi-fuzzy quantifiers
can be defined by means of the boolean combinations $\Phi_{1}\left(
Y_{1},Y_{2}\right)  =Y_{1}\cap Y_{2}$ and $\Phi_{2}\left(  Y_{1}%
,\overline{Y_{2}}\right)  =Y_{1}\cap \overline{Y_{2}}$.  

Let $Q:\mathcal{P}\left(  E\right)  ^{n}\longrightarrow\mathbf{I}$ be a
quantitative semi-fuzzy quantifier $Q:\mathcal{P}\left(  E\right)
^{n}\longrightarrow\mathbf{I}$ on a finite base set $E\neq\varnothing$, and
let us suppose it can be expressed following expression
\ref{EqTeoremaCuantitativo} for some set $\Phi_{1}\left(  Y_{1},\ldots
,Y_{n}\right)  ,\ldots,\Phi_{K}\left(  Y_{1},\ldots,Y_{n}\right)  $ of boolean
combinations and some $q:\left\{  0,\ldots,m\right\}  ^{K}\longrightarrow
\mathbf{I}$. For convenience, we will define $q^{\prime}:\left[  0,1\right]
^{K}\longrightarrow\mathbf{I}$ such that:%
\begin{equation*}
q^{\prime}\left(  \frac{\left\vert \Phi_{1}\left(  Y_{1},\ldots,Y_{n}\right)
	\right\vert }{m},\ldots,\frac{\left\vert \Phi_{K}\left(  Y_{1},\ldots
	,Y_{n}\right)  \right\vert }{m}\right)  = \\
q\left(  \left\vert \Phi_{1}\left(
Y_{1},\ldots,Y_{n}\right)  \right\vert ,\ldots,\left\vert \Phi_{K}\left(
Y_{1},\ldots,Y_{n}\right)  \right\vert \right)
\end{equation*}
$q^{\prime}:\left[  0,1\right]  ^{K}\longrightarrow\mathbf{I}$ simply
normalizes $q$ in the interval of proportions $\left[  0,1\right]  ^{K}$\footnote{For simplicity of the notation, we will use $\left[  0,1\right]
^{K}$ instead of $\left\{  0,\frac{1}{m},\ldots,\frac{m-1}{m},1\right\}^{K}  $. }.

We introduce now the definition of
the \textit{poisson binomial distribution}. Let us consider a sequence of $m$
\textit{independent bernoulli trials} $\mathbf{B}=P_{1},\ldots,P_{m}$ that are
not necessarily identically distributed. Let be $p_{1},\ldots,p_{m}$ the
corresponding probabilities of the independent bernouilli trials. The
probability function of the poisson binomial distribution is:%

\[
\Pr^{\mathbf{B}}\left(  K=k\right)  =\sum_{A\in F_{k}}%
{\displaystyle\prod\limits_{i\in A}}
p_{i}%
{\displaystyle\prod\limits_{j\in A^{c}}}
\left(  1-p_{j}\right)
\]
where $F_{k}$ is the set of all subsets of $k$ integers that can be selected
from $\left\{  1,2,3,\ldots,m\right\}  $.

We now introduce a notation for representing the poisson bernoulli succession
$\mathbf{B}=P_{1},\ldots,P_{m}$ with probabilities $p_{1},\ldots,p_{m}$ by
means of a fuzzy set:

\begin{notation}
Let be $\mathbf{B}=P_{1},\ldots,P_{m}$ a poisson bernoulli succession with
probabilities $p_{1},\ldots,p_{m}$. We will denote by $X^{\mathbf{B}}%
\in\widetilde{\mathcal{P}}\left(  E\right)  $ the fuzzy set defined in the
following way:%
\[
\mu_{X^{\mathbf{B}}}\left(  e_{i}\right)  =p_{i}%
\]

\end{notation}

Under the probabilistic interpretation of the $\mathcal{F}^{A}$ QFM, a crisp
set $Y\in\mathcal{P}\left(  E\right)  $ can be interpreted as a realization of
a poisson bernoulli succession $\mathbf{B}=P_{1},\ldots,P_{m}$ with
probabilities $p_{1},\ldots,p_{m}$ such that $\chi_{Y}\left(  e_{i}\right)
=P_{i}$\footnote{By $\chi_{Y}\left(  e_{i}\right)  $ we are representing the
characteristic function of $Y$; that is: $\chi_{Y}\left(  e_{i}\right)  =1$ if
$e_{i}\in Y$ and $0$ otherwise.}. In this sense,%
\begin{align*}
\overset{\mathbf{B}}{\Pr}\left(  Y\right)   &  =%
{\displaystyle\prod_{i|e_{i}\in Y}}
p_{i}%
{\displaystyle\prod_{j|e_{j}\notin Y}}
\left(  1-p_{j}\right) 
  =%
{\displaystyle\prod_{i|e_{i}\in Y}}
\left(  P_{i}=1\right)
{\displaystyle\prod_{j|e_{j}\notin Y}}
\left(  P_{j}=0\right) 
  =m_{X^{\mathbf{B}}}\left(  Y\right).
\end{align*}

Now, we will compute the projection of the probability function used in the
definition of the $\mathcal{F}^{A}$ QFM for the cardinalities of each possible
boolean combination associated to a quantitative semi-fuzzy quantifier
$Q:\mathcal{P}\left(  E\right)  ^{n}\longrightarrow\mathbf{I}$. As $Q$ is
quantitative, by theorem \ref{TeoremaCuantitativo} it can be defined by means
of a function $q:\left\{  0,\ldots,\left\vert E\right\vert \right\}
^{K}\longrightarrow\mathbf{I}$ depending on the\ cardinalities of the boolean
combinations of the input sets ($\left\vert \Phi_{1}\left(  Y_{1},\ldots
,Y_{n}\right)  \right\vert ,\ldots,\left\vert \Phi_{K}\left(  Y_{1}%
,\ldots,Y_{n}\right)  \right\vert $). Then:%
\begin{align*}
\mathcal{F}^{A}\left(  Q\right)  \left(  X_{1},\ldots,X_{n}\right) 
 & =\sum_{Y_{1}\in\mathcal{P}\left(  E\right)  }\ldots\sum_{Y_{n}%
\in\mathcal{P}\left(  E\right)  }m_{X_{1}}\left(  Y_{1}\right)  \ldots
m_{X_{n}}\left(  Y_{n}\right)  Q\left(  Y_{1},\ldots,Y_{n}\right) \\
&  =\sum_{\substack{\left(  i_{1},\ldots,i_{K}\right)  \in \\ \left\{  0,\ldots,m\right\}^{K}}}
\sum_{\substack{Y_{1},\ldots,Y_{n}\in\mathcal{P}\left(  E\right)
\; | \\
\left\vert \Phi_{1}\left(  Y_{1},\ldots,Y_{n}\right)  \right\vert
=i_{1}\wedge\\\ldots\\\left\vert \Phi_{k}\left(  Y_{1},\ldots,Y_{n}\right)
\right\vert =i_{K}}} \; m_{X_{1}}\left(  Y_{1}\right)  \ldots m_{X_{n}}\left(
Y_{n}\right)  \times\\
& \qquad q\left(  \left\vert \Phi_{1}\left(  Y_{1},\ldots,Y_{n}\right)  \right\vert
,\ldots,\left\vert \Phi_{K}\left(  Y_{1},\ldots,Y_{n}\right)  \right\vert
\right) \\
&  =\sum_{\substack{\left(  i_{1},\ldots,i_{K}\right)  \in \\ \left\{  0,\ldots,m\right\}
^{K}}} q\left(  i_{1},\ldots,i_{K}\right)  
\sum_{\mathclap{\substack{Y_{1},\ldots,Y_{n}%
\in\mathcal{P}\left(  E\right) \; |\\\left\vert \Phi_{1}\left(  Y_{1}%
,\ldots,Y_{n}\right)  \right\vert =i_{1}\wedge\\\ldots\\\left\vert \Phi
_{k}\left(  Y_{1},\ldots,Y_{n}\right)  \right\vert =i_{K}}}} 
m_{X_{1}}\left(Y_{1}\right)  \ldots m_{X_{n}}\left(  Y_{n}\right).
\end{align*}

Let us denote by%
\begin{equation}
f\left(  i_{1},\ldots,i_{K}\right) = 
\sum_{\mathclap{\substack{Y_{1},\ldots,Y_{n}%
\in\mathcal{P}\left(  E\right) \; |\\\left\vert \Phi_{1}\left(  Y_{1}%
,\ldots,Y_{n}\right)  \right\vert =i_{1}\wedge\\\ldots\\\left\vert \Phi
_{k}\left(  Y_{1},\ldots,Y_{n}\right)  \right\vert =i_{K}}}} \; m_{X_{1}}\left(
Y_{1}\right)  \ldots m_{X_{n}}\left(  Y_{n}\right)  \label{EqProbabilityMx}%
\end{equation}%
where $f\left(  i_{1},\ldots,i_{K}\right)  $ is a probability function. Take into
account that $m_{X_{1}}\left(  Y_{1}\right)  \ldots m_{X_{n}}\left(
Y_{n}\right)  $ define a probility over $\left(  Y_{1},\ldots,Y_{n}\right)
\in\mathcal{P}\left(  E\right)  ^{n}$, and $f\left(  i_{1},\ldots
,i_{K}\right)  $ simply distributes the probabilities of $\left(  Y_{1}%
,\ldots,Y_{n}\right)  \in\mathcal{P}\left(  E\right)  ^{n}$ over the
cardinalities of the $K$ boolean combinations.

\begin{theorem}
Let $f\left(  i_{1},\ldots,i_{K}\right)  $ be the probability distribution
that is obtained when we compute the probability induced by the $X_{1}%
,\ldots,X_{n}\in\mathcal{P}\left(  E\right)  ^{n}$ fuzzy sets over the
cardinalities of the boolean combinations $\left\vert \Phi_{1}\left(
Y_{1},\ldots,Y_{n}\right)  \right\vert ,\ldots,\left\vert \Phi_{K}\left(
Y_{1},\ldots,Y_{n}\right)  \right\vert $ following equation
\ref{EqProbabilityMx}. The probability projection $j$ of $f\left(
i_{1},\ldots,i_{K}\right)  $ will follow a poisson binomial distribution of
parameters:%
\begin{align*}
p_{1}^{j}  &  =\mu_{X_{1}^{\left(  l_{j,1}\right)  }\widetilde{\cap}%
\ldots\widetilde{\cap}X_{n}^{\left(  l_{j,n}\right)  }}\left(  e_{1}\right) \\
&  \ldots\\
p_{m}^{j}  &  =\mu_{X_{1}^{\left(  l_{j,1}\right)  }\widetilde{\cap}%
\ldots\widetilde{\cap}X_{n}^{\left(  l_{j,n}\right)  }}\left(  e_{m}\right)
\end{align*}
where $X_{1}^{\left(  l_{j,1}\right)  }\widetilde{\cap}\ldots\widetilde{\cap
}X_{n}^{\left(  l_{j,n}\right)  }=\Phi_{j}\left(  X_{1},\ldots,X_{n}\right)  $
is the j-th boolean combination.
\end{theorem}

\begin{proof}
We will only give an intuitive idea of this result. In appendix
\ref{AnnexProof} an analytical proof can be consulted.\newline By assumption,
the $\mathcal{F}^{A}$ QFM is interpreting membership grades of the input sets
as probabilities, and considering that the independence assumption is always
fulfilled between different elements and sets. The probability projection $j$
of $f\left(  i_{1},\ldots,i_{K}\right)  $ simply denotes the probability of
the different cardinalities of one of these boolean combinations. But the
probability of an element $e_{s}$ of pertaining to the boolean combination
$\Phi_{j}\left(  Y_{1},\ldots,Y_{n}\right)  $ is just the probability of
$e_{s}$ pertaining to every fuzzy set $X_{r}^{\left(  l_{r}\right)  }$ such
that $l_{r}=1$ and non pertaining to every fuzzy set $X_{r}^{\left(
l_{r}\right)  }$ set such that $l_{r}=0$. As this is fulfilled for every $e\in
E$, the cardinality of the boolean combination follows a poisson binomial
distribution with the indicated parameters.
\end{proof}

\begin{proposition}
\textit{\label{LimitCase}}Let $Q:\mathcal{P}\left(  E\right)  ^{n}%
\longrightarrow\mathbf{I}$ be a semi-fuzzy quantitative quantifier on a finite
base set $E\neq\varnothing$, $\Phi_{1}\left(  Y_{1},\ldots,Y_{n}\right)
,\ldots,\Phi_{K}\left(  Y_{1},\ldots,Y_{n}\right)  $, $K\in\mathbb{N}$ boolean
combinations, and $q:\left\{  0,\ldots,m\right\}  ^{K}\longrightarrow
\mathbf{I}$ the corresponding function for which:%
\begin{align*}
Q\left(  Y_{1},\ldots,Y_{n}\right) 
&  = q\left(  \left\vert \Phi_{1}\left(
Y_{1},\ldots,Y_{n}\right)  \right\vert ,\ldots,\left\vert \Phi_{K}\left(
Y_{1},\ldots,Y_{n}\right)  \right\vert \right) \\
&  = q^{\prime}\left(  \frac{\left\vert \Phi_{1}\left(  Y_{1},\ldots
	,Y_{n}\right)  \right\vert }{m},\ldots,\frac{\left\vert \Phi_{K}\left(
	Y_{1},\ldots,Y_{n}\right)  \right\vert }{m}\right)
\end{align*}
If $q^{\prime}:\left[  0,1\right]  ^{K}\longrightarrow\mathbf{I}$ is
continuous around
\[
\left(  \frac{\sum_{i=1}^{m}\mu_{\Phi_{1}\left(  X_{1},\ldots,X_{n}\right)  }%
}{m},\ldots,\frac{\sum_{i=1}^{m}\mu_{\Phi_{K}\left(  X_{1},\ldots
,X_{n}\right)  }}{m}\right)
\]
then the following result will be fulfilled when the size of $E$ tend to infinite:
\end{proposition}

%

\begin{align*}
\lim_{\left\vert E\right\vert \rightarrow\infty}\mathcal{F}^{A}\left(
Q\right)  \left(  X_{1},\ldots,X_{n}\right) 
& = q^{\prime}\left(  \frac
{\sum_{i=1}^{m}\mu_{\Phi_{1}\left(  X_{1},\ldots,X_{n}\right)  }}{m}%
,\ldots,\frac{\sum_{i=1}^{m}\mu_{\Phi_{K}\left(  X_{1},\ldots,X_{n}\right)  }%
}{m}\right)
\end{align*}

Before proving proposition \ref{LimitCase}, we would like to make some
appointments about the applicability of the result. In general, we always
could find a $q^{\prime}$ continuous around $\left(  \frac{\sum_{i=1}^{m}%
\mu_{\Phi_{1}\left(  X_{1},\ldots,X_{n}\right)  }}{m},\ldots,\frac{\sum
_{i=1}^{m}\mu_{\Phi_{K}\left(  X_{1},\ldots,X_{n}\right)  }}{m}\right)  $ such
that previous result would be applicable. But in choosing a `proportional
expression' for $q^{\prime}$, we are indicating that the types of fuzzy
quantifiers in which we are mainly interested are `proportional quantifiers'.
In practical applications, support functions associated to proportional
quantifiers are generally defined by means of `smooth'  fuzzy numbers over
$\left[  0,1\right]  $, which guarantees a good approximation when the size of
the referential set is sufficiently large. 

\begin{proof}
Let $f\left(  i_{1},\ldots,i_{K}\right)  $ be the probability distribution
that is obtained when we compute the probability induced by the $X_{1}%
,\ldots,X_{n}\in\mathcal{P}\left(  E\right)  ^{n}$ fuzzy sets over the
cardinalities of the boolean combinations $\left\vert \Phi_{1}\left(
Y_{1},\ldots,Y_{n}\right)  \right\vert ,\ldots,\left\vert \Phi_{K}\left(
Y_{1},\ldots,Y_{n}\right)  \right\vert $. We know that the probability
projection $f^{j}\left(  i_{s}\right)  $ follows a poisson binomial
distribution of parameters
\begin{align*}
p_{1}^{j}  &  =\mu_{X_{1}^{\left(  l_{j,1}\right)  }\widetilde{\cap}%
\ldots\widetilde{\cap}X_{n}^{\left(  l_{j,n}\right)  }}\left(  e_{1}\right) \\
&  \ldots\\
p_{m}^{j}  &  =\mu_{X_{1}^{\left(  l_{j,1}\right)  }\widetilde{\cap}%
\ldots\widetilde{\cap}X_{n}^{\left(  l_{j,n}\right)  }}\left(  e_{m}\right)
\end{align*}
Moreover,%
\begin{align*}
\mathcal{F}^{A}\left(  Q\right)  \left(  X_{1},\ldots,X_{n}\right) 
&  =\sum_{\substack{\left(  i_{1},\ldots,i_{k}\right) \in \\ \left\{  0,\ldots,m\right\}^{K}}} q\left(i_{1},\ldots,i_{K}\right)  
\sum_{\mathclap{\substack{Y_{1},\ldots,Y_{n} \in\mathcal{P}\left(  E\right) \; | \\\left\vert \Phi_{1}\left(  Y_{1} ,\ldots,Y_{n}\right)  \right\vert =i_{1}\wedge\\\ldots\\\left\vert \Phi_{k}\left(  Y_{1},\ldots,Y_{n}\right)  \right\vert =i_{K}}}} \; m_{X_{1}}\left(Y_{1}\right)  \ldots m_{X_{n}}\left(  Y_{n}\right) \\
&  =\sum_{\substack{\left(  i_{1},\ldots,i_{k}\right)  \in \\ \left\{  0,\ldots,m\right\}^{K}}} q^{\prime}\left(  \frac{i_{1}}{m},\ldots,\frac{i_{K}}{m}\right)  f\left(
i_{1},\ldots,i_{K}\right)
\end{align*}
Let $f^{\prime}:\left[  0,1\right]  ^{K}\longrightarrow\mathbf{I}$ be
probability distribution defined by:%
\[
f^{\prime}\left(  s_{1},\ldots,s_{K}\right)  =f\left(  m\times i_{1}%
,\ldots,m\times i_{K}\right)
\]
that normalizes $f$ in the interval $\left[  0,1\right]  ^{K}$. Then,%
\begin{equation*}
\mathcal{F}^{A}\left(  Q\right)  \left(  X_{1},\ldots,X_{n}\right)   \\
=\sum_{\left(  i_{1},\ldots,i_{K}\right)  \in m^{K}}q^{\prime}\left(
\frac{i_{1}}{m},\ldots,\frac{i_{K}}{m}\right)  f^{\prime}\left(  \frac{i_{1}%
}{m},\ldots,\frac{i_{K}}{m}\right)
\end{equation*}

As we are normalizing $f$ by $m$, the corresponding $f^{j\prime}\left(
i_{s}\right)  $ projection of $f^{\prime}$ will follow a probability
distribution such that:%
\begin{align*}
average\left(  f^{j\prime}\right)   &  =\frac{average\left(  f^{j}\right)
}{m}=\frac{\sum_{i=1}^{m}p_{i}^{j}}{m}\\
var\left(  f^{j\prime}\right)   &  =\frac{1}{m^{2}}var\left(  f^{j}\right)
=\frac{1}{m^{2}}\sum_{i=1}^{m}p_{i}^{j}\left(  1-p_{i}^{j}\right)
\end{align*}
but when $m\longrightarrow\mathbf{\infty}$ the variance tends to $0$.

And as the variance tends to $0$, $f^{j\prime}\overset{p}{\longrightarrow
}\frac{\sum_{i=1}^{m}p_{i}^{j}}{m}$, and as $q^{\prime}\left(  s_{1}%
,\ldots,s_{K}\right)  $ is continuous around $\left(  \frac{\sum_{i=1}%
^{m}p_{i}^{1}}{m},\ldots,\frac{\sum_{i=1}^{m}p_{i}^{K}}{m}\right)  $, by continuous mapping theorem
\footnote{Take into account that, as the variance tends to $0$, by the
\textit{Chebyshev inequality} we always could find an interval around
$average\left(  f^{j\prime}\right)  $ as small and containing a probability mass as high as desired for any $j$. This will allow to put as much probability
around $\left(  \frac{\sum_{i=1}^{m}p_{i}^{1}}{m},\ldots,\frac{\sum_{i=1}%
^{m}p_{i}^{k}}{m}\right)  $ as we wanted, where $q^{\prime}$ is continuous by
hypothesis.}:%

\begin{align}
&  \lim_{m\rightarrow\infty}\sum_{\left(  i_{1},\ldots,i_{k}\right)  \in
m^{k}}q^{\prime}\left(  \frac{i_{1}}{m},\ldots,\frac{i_{K}}{m}\right)
f\left(  \frac{i_{1}}{m},\ldots,\frac{i_{K}}{m}\right) \nonumber 
\overset{p}{\longrightarrow}q^{\prime}\left(  \frac{\sum_{i=1}^{m}p_{i}^{1}}{m}%
,\ldots,\frac{\sum_{i=1}^{m}p_{i}^{K}}{m}\right) \label{EqApproximation}\\
&  =q^{\prime}\left(  \frac{\sum_{i=1}^{m}\mu_{\Phi_{1}\left(  X_{1}%
,\ldots,X_{n}\right)  }}{m},\ldots,\frac{\sum_{i=1}^{m}\mu_{\Phi_{K}\left(
X_{1},\ldots,X_{n}\right)  }}{m}\right). \nonumber
\end{align}
\end{proof}

This result guarantees that the $\mathcal{F}^{A}$ QFM converges to the Zadeh's
model for unary proportional and binary proportional quantifiers when the size
of the referential set tends to infinite and the intersection is modelled by
means of the \textit{product tnorm} in the proportional case, as these
quantifiers basically depend on\footnote{Take into account that for
proportional quantifiers $\frac{\left\vert Y_{1}\cap Y_{2}\right\vert
}{\left\vert Y_{1}\right\vert }=\frac{\left\vert Y_{1}\cap Y_{2}\right\vert
}{m}/\left(  \frac{\left\vert Y_{1}\cap Y_{2}\right\vert }{m}+\frac{\left\vert
Y_{1}\cap\overline{Y_{2}}\right\vert }{m}\right)  $. In this case, the
$\mathcal{F}^{A}$ QFM will converge to $f_{Q}\left(  \frac{\sum_{e\in E}%
\mu_{X_{1}}\left(  e\right)  \mu_{X_{2}}\left(  e\right)  }{\sum_{e\in E}%
\mu_{X_{1}}\left(  e\right)  }\right)  $.}:%
\[%
\begin{array}
[c]{ccc}%
q:\frac{\left\vert Y\right\vert }{m}\longrightarrow\mathbf{I} & \text{:} &
\text{unary quantifiers}\\
q:\left(  \frac{\left\vert Y_{1}\cap Y_{2}\right\vert }{m},\frac{\left\vert
Y_{1}\cap\overline{Y_{2}}\right\vert }{m}\right)  ,\longrightarrow\mathbf{I} &
\text{:} & \text{binary quantifiers}%
\end{array}
\]

As we introduced below, the normalization by $m$ is coherent with proportional
linguistic quantifiers, that are generally defined by means of `smooth' fuzzy
numbers in $\left[  0,1\right]  $. In these situations, the result guarantees
that the probability of the projections of $f^{\prime}\left(  s_{1}%
,\ldots,s_{s}\right)  $ will concentrate around the average of the projections
as we increase the size of the referential set. As a consequence, if the
variation of the fuzzy number that supports the linguistic quantifier is small
around this average, we could expect a good approximation of the
$\mathcal{F}^{A}$ QFM using \ref{EqApproximation}\ when the size of the
referential set tends to infinite.

\section{Quality of the convergence and Monte Carlo approximation of the
$\mathcal{F}^{A}$ QFM}

In section \ref{ComputationalAlgorithms} we will present some computational
exact implementations of the $\mathcal{F}^{A}$ QFM for evaluating the most
common linguistic quantifiers. We advance that the complexity of the exact
implementation of the $\mathcal{F}^{A}$ QFM is $O\left(  m^{2}\right)  $ for
unary quantifiers, $O\left(  m^{3}\right)  $ for binary proportional
quantifiers and $O\left(  m^{r+1}\right)  $ in the general case, being $r$ the
number of boolean combinations that are necessary for the definition of the
semi-fuzzy quantifier. For some applications, and specifically for quantifiers
depending on a high value of $r$, this complexity could be too high for
applying the model to big fuzzy sets.

One consequence of the result of the previous section is that the
$\mathcal{F}^{A}$ QFM can be approximated in linear time for fuzzy sets
containing a sufficiently large number of elements. But we do not know if the
proposed approximation is sufficiently accurate for problems where the exact
implementation could not be applied due to its computational demands. We will make
now a deeper analysis about the applicability of the results of previous
section for approximating the $\mathcal{F}^{A}$ QFM, connecting them with a
proposal to use a Monte Carlo simulation. Let us consider the following example:

\begin{example}
\label{ExampleConvergence}Let us consider a fuzzy set $X=\left\{
0.5/e_{1},\ldots,0.5/e_{m}\right\}  $. In this situation, the probability
distribution subjacent to the $\mathcal{F}^{A}$ QFM is a binomial distribution
with parameters $\left(  m,0.5\right)  $. Let us consider a trapezoidal
function $T_{0.5,0.6,\infty,\infty}\left(  x\right)  $ and the unary
semi-fuzzy quantifier defined as $Q\left(  Y\right)  =T_{0.5,0.6,\infty
,\infty}\left(  \left\vert Y\right\vert \right)  $. The following table
compares the result of the application of the $\mathcal{F}^{A}$ QFM\thinspace
with its approximation by means of the Zadeh's model:\newline%
\begin{center}
\small
\begin{tabular}[c]{l c c}
\toprule
$m,X$ & $F^{A}\left(  X\right)  $ & $f_{Q}\left(\overline{X}\right)  $\\
\midrule
$50,X=\left\{  \underset{50}{\underbrace{0.5,\ldots,0.5}}\right\}$ & $0.260$ & $0$\\
\midrule
$100,X=\left\{  \underset{100}{\underbrace{0.5,\ldots,0.5}}\right\}$ & $0.195$ & $0$\\
\midrule
$500,X=\left\{  \underset{500}{\underbrace{0.5,\ldots,0.5}}\right\}$ & $0.089$ & $0$\\
\bottomrule
\end{tabular}
\end{center}
\vspace{4mm}

\end{example}

Previous example proves that, even for a large fuzzy set containing 500
elements, the error of the approximation is not negligable for a semi-fuzzy
quantifier defined by means of a fuzzy number that seems very plausible from a
practical viewpoint. Moreover, the error will be greater for a semi-fuzzy
quantifier definfed by means of a fuzzy number with a higher slope.

We will now introduce a theorem applicable to the poisson binomial
distribution \cite[page 263]{Degroot88}.

\begin{theorem} Central limit theorem applied to Bernoulli variables. Let
$X_{1},\ldots,X_{m}$ be independent random variables, each $X_{i}$ following a
Bernoulli distribution with parameter $p_{i}$. Moreover, let us suppose that
the infinite sum $\sum_{i=1}^{\infty}p_{i}\left(  1-p_{i}\right)  $ is
divergent and let $Y_{m}$ be%
\[
Y_{m}=\frac{\sum_{i=1}^{m}X_{i}-\sum_{i=1}^{n}p_{i}}{\left(  \sum_{i=1}%
^{m}p_{i}q_{i}\right)  ^{1/2}}.%
\]
Then
\[
\lim_{n\rightarrow\infty}\Pr\left(  Y_{m}\leq x\right)  =\Phi\left(  x\right)
\]
where $\Phi\left(  x\right)  $ is the standard normal distribution function.
\end{theorem}

In practical situations, this result allow us to approximate a poisson
binomial distribution by a normal distribution when the variance of the
distribution is high (take into account that we are interpreting the
cardinality of a fuzzy set as a poisson binomial distribution). Cases of a low
variance for poisson binomial distributions with a high number of parameters
will be associated to situations in which most parameters are really close to
0 or 1\footnote{For many quantifiers, results of the $\mathcal{F}^{A}$ QFM and
of the Zadeh's model will be extremely close even for small fuzzy sets. There
are two main reasons for that. Once is that the variance of the probability
projections associated to the different boolean combinations was very low and
as a consequence, that the probability distributions would be very
concentrated around the average. The other situation is that the fuzzy number
used in the definition of the semi-fuzzy quantifier was approximately linear
in the area in which much of the probability is concentrated. In this
situation, the symmetry of the normal distribution (to which the poisson
binomial distribution converges) will cause that the result of the evaluation
will be really close to the result of the Zadeh's model.}. In these cases, the
approximation by means of the normal distribution will be poor, but the
probability distribution will be extremely concentrated around the average,
which will guarantee an even better approximation by means of Montecarlo.

Let us consider again a fuzzy set $X=\left\{  0.5/e_{1},\ldots,0.5/e_{m}%
\right\}  $ whose underlying probability distribution following the
$\mathcal{F}^{A}$ QFM interpretation is a binomial distribution with
parameters $\left(  m,0.5\right)  $. We will compute the confidence intervals
for the $0.95$ and $0.99$ probability mass approximating the underlying
probability of the $\mathcal{F}^{A}$ QFM by means of a normal distribution.

\begin{example}
The following table shows the confidence intervals for the underlying
probability distribution of a fuzzy set $X=\left\{  0.5/e_{1},\ldots
,0.5/e_{m}\right\}  $, following the $\mathcal{F}^{A}$ QFM
interpretation:\newline%
\begin{center}
\small
\begin{tabular}[c]{l l l l }
\toprule
$m$ & $\overline{X}$ & $\frac{X}{m},0.95$ & $\frac{X}{m},0.99$\\
\midrule
$50$ & $25$ & $\left(  0.36,0.64\right)  $ & $\left(  0.32,0.68\right)$\\
\midrule
$100$ & $50$ & $\left(  0.40,0.60\right)  $ & $\left(  0.37.0.63\right)$\\
\midrule
$1000$ & $500$ & $\left(  0.47,0.53\right)  $ & $\left(  0.45,0.54\right)$\\
\midrule
$10000$ & $5000$ & $\left(  0.49.0.51\right)  $ & $\left(  0.49,0.51\right)$\\
\bottomrule
\end{tabular}
\end{center}
\vspace{4mm}

\end{example}

Previous example shows that the probability distribution is really
concentrated around the average for medium size fuzzy sets. In previous
example, we have chosen the binomial distribution of parameter $0.5$ as it is
the highest variance distribution in the family of poisson binomial 
distributions. Take into account that for a poisson binomial distribution
$\mathbf{B}$, $var\left(  \mathbf{B}\right)  =\sum_{i=1}^{m}p_{i}\left(
1-p_{i}\right)  $, and that the maximum of $p_{i}\left(  1-p_{i}\right)  \,$is
obtained for $p_{i}=0.5$.

The idea of the Monte Carlo simulation is simply to generate, for each $X_{i}%
$, a random binary vector using a Bernoulli trial of probability $\mu_{X_{i}%
}\left(  j\right)  $ for each $e_{j}$. Previous example indicates that the
$f^{j\prime}\left(  i_{s}\right)  $ projections of $f^{\prime}$ would be very
concentrated around the average when the size of the referential set contains
a large number of elements, which will allow to expect a really good
approximation of the $\mathcal{F}^{A}$ QFM by means of a Monte Carlo
simulation. Moreover, a Monte Carlo simulation can be easily parallelized. In section \ref{MonteCarloApproximation} the algorithm for unary quantifiers is
presented. The extension to higher arity quantifiers is trivial.

\section{Efficient implementation of the $\mathcal{F}^{A}$
model\label{ComputationalAlgorithms}}

For quantititative quantiers is possible to develop polynomial algorithms for
the $\mathcal{F}^{A}$ DFS. Let us remember that the class of quantitative
quantifiers is composed of the semi-fuzzy quantifiers that are invariant under
automorphims \cite[section 4.13]{Glockner06Libro}, and that they can be
expressed as a function of the cardinalities of their arguments and their
boolean combinations. The class of quantitative quantifiers include the most
interesting ones for applications, and in particular the common
\textit{absolute, proportional and comparative} quantifiers.

\subsection{Quantitative unary quantifiers}

Let $Q:\mathcal{P}\left(  E\right)  \rightarrow\mathbf{I}$ be an unary
semi-fuzzy quantifier defined over a referential set $E^{m}=\left\{
e_{1},\ldots,e_{m}\right\}  $. Quantitative unary semi-fuzzy quantifiers can
always be expressed by means of a function $q:\left\{  0,\ldots,\left\vert
E\right\vert \right\}  \rightarrow\mathbf{I}$
(theorem\ \ref{TeoremaCuantitativo}); that is, a function that goes from
cardinality values in $\mathbf{I}$. In this way, there exists $q$ such that
$q\left(  j\right)  =Q\left(  Y_{j}\right)  $ where $Y_{j}\in\mathcal{P}%
\left(  E\right)  $ is an arbitrary set of cardinality $j$ ($\left\vert
Y_{j}\right\vert =j$).

Let $X\in\mathcal{P}\left(  E\right)  $ be a fuzzy set. Then,%

\begin{align*}
\mathcal{F}^{A}\left(  Q\right)  \left(  X\right) 
& =\sum_{Y \in \mathcal{P}\left(  E\right)  }m_{X}\left(  Y\right)  Q\left(  Y\right) \\
&  =\sum_{\substack{Y \in \mathcal{P}\left(  E\right) \\ | \; \left\vert Y\right\vert =0}} m_{X}\left(  Y\right)  Q\left(  Y\right) + \ldots + \sum_{\substack{Y\in\mathcal{P}\left(E\right) \\ | \; \left\vert Y\right\vert =m}} m_{X}\left(  Y\right)  Q\left(  Y\right) \\
&  =\sum_{\substack{Y \in \mathcal{P}\left(  E\right) \\ | \; \left\vert Y\right\vert =0}} m_{X}\left(  Y\right)  q\left(  0\right) + \ldots + \sum_{\substack{Y\in\mathcal{P}\left(E\right) \\ | \; \left\vert Y\right\vert =m}} m_{X}\left(  Y\right)  q\left(  m\right) \\
&  =\sum_{j=0}^{m}\Pr\left(  card_{X}=j\right)  q\left(  j\right)
\end{align*}

The algorithm we will present uses the fact that it is possible to compute the
probability $\Pr_{E^{m}}\left(  card_{X}=j\right)  ,$ $j=0,\ldots,m$ for a
referential set $E^{m}$ of $m$ elements using the probabilities $\Pr_{E^{m-1}%
}\left(  card_{X^{E^{m-1}}}=j\right)  ,j=0,\ldots,m-1$ where $E^{m-1}=\left\{
e_{1},\ldots,e_{m-1}\right\}  $ and $X^{E^{m-1}}$ is the projection of $X$
over $E^{m-1}$ (that is, the fuzzy set $X$ without the element $e_{m}$). In
this way, it is easy to develop a recursive function for computing the
probabilities of the cardinalities in $E^{m}$.

In the case of a referential set of one element ($E^{1}=\left\{
e_{1}\right\}  $) the probabilities of the cardinalities of a fuzzy set
$X\in\mathcal{P}\left(  E^{1}\right)  $ are simply:%
\begin{align*}
\Pr\left(  card_{X}=0\right)  & = m_{X}\left(\varnothing\right)  =1-\mu_{X}\left(  e_{1}\right)\\
\Pr\left(  card_{X}=1\right)  & = m_{X}\left(\left\{  e_{1}\right\}  \right)  =\mu_{X}\left(  e_{1}\right)
\end{align*}

Let us suppose now a referential set of $m+1$ elements ($E^{m+1}=\left\{
e_{1},\ldots,e_{m+1}\right\}  $), let $X\in\widetilde{\mathcal{P}}\left(
E^{m+1}\right)  $ be a fuzzy set on $E^{m+1}$, $E^{m}=\left\{  e_{1}%
,\ldots,e_{m}\right\}  $ and $X^{E^{m}}\in\widetilde{\mathcal{P}}\left(
E^{m}\right)  $ the projection of $X$ in $E^{m}$; that is, $\mu_{X^{E^{m}}%
}\left(  e_{j}\right)  =\mu_{X}\left(  e_{j}\right)  ,1\leq j\leq m$.
Moreover, let us suppose we know the probabilities of the cardinalities
associated to $X^{E^{m}}$ ($\Pr\left(  card_{X^{E^{m}}}=0\right)  ,\ldots
,\Pr\left(  card_{X^{E^{m}}}=m\right)  $). Now, we will compute the
probabilities of $X$ using the probabilities of the cardinalities on
$X^{E^{m}}$:

\noindent \textbf{Case 1:} $\Pr\left(  card_{X}=0\right)  $%
\begin{align*}
\Pr\left(  card_{X}=0\right) & = \sum_{\substack{Y\in\mathcal{P}\left(  E^{m+1} \right) | \; \left\vert Y\right\vert =0}} m_{X}\left(  Y\right)  \\
&  = m_{X}\left(\varnothing\right) \\
&  = \left(  1-\mu_{X}\left(  e_{1}\right)  \right)  \ldots\left(  1-\mu
_{X}\left(  e_{m}\right)  \right)  \left(  1-\mu_{X}\left(  e_{m+1}\right)
\right) \\
&  = m_{X^{E^{m}}}\left(  \varnothing\right)  \left(  1-\mu_{X}\left(e_{m+1}\right)  \right) \\
&  = \Pr\left(  card_{X^{E^{m}}}=0\right)  \left(  1-\mu_{X}\left(e_{m+1}\right)  \right)
\end{align*}

\noindent \textbf{Case 2:} $\Pr\left(  card_{X}=m+1\right)  $%
\begin{align*}
\Pr\left(  card_{X}=m+1\right) & =\sum_{\substack{Y\in\mathcal{P}\left(E^{m+1}\right)  | \; \left\vert Y\right\vert =m+1}} m_{X}\left(  Y\right) \text{\hspace{18mm}} \\
&  = m_{X}\left(  E^{m+1}\right) \\
& = \mu_{X}\left(  e_{1}\right)  \ldots\mu_{X}\left(  e_{m}\right)  \mu_{X}\left(  e_{m+1}\right) \\
& = m_{X^{E^{m}}}\left(  E^{m}\right)  \mu_{X}\left(  e_{m+1}\right) \\
& = \Pr\left(  card_{X^{E^{m}}}=m\right)  \mu_{X}\left(  e_{m+1}\right)
\end{align*}

\noindent \textbf{Case 3:} \textbf{\ }$\Pr\left(  card_{X}=j\right)  ,0<j<m+1$%
\begin{align*}
\Pr\left(  card_{X}=j\right) & =\sum_{\substack{Y\in\mathcal{P}\left(  E^{m+1}\right) \\ | \; \left\vert Y\right\vert =j}} m_{X}\left(  Y\right) \\
&  = \sum_{\substack{Y\in\mathcal{P}\left(  E^{m+1}\right) \\ | \; \left\vert Y\right\vert =j\wedge e_{m+1}\notin Y}} m_{X}\left(  Y\right) + \sum_{\substack{Y\in\mathcal{P}\left(E^{m+1}\right) \\ | \; \left\vert Y\right\vert =j\wedge e_{m+1}\in Y}} m_{X}\left( Y\right) \\
&  = \sum_{\substack{Y\in\mathcal{P}\left(  E^{m}\right) \\ | \; \left\vert Y\right\vert =j}} m_{X^{E^{m}}}\left(  Y\right)  \left(  1-\mu_{X}\left(  e_{m+1}\right) \right)  + \ldots 
  + \sum_{\substack{Y\in\mathcal{P}\left(  E^{m}\right) \\ | \; \left\vert Y\right\vert =j-1}} m_{X}\left(  Y\right)  \mu_{X}\left(  e_{m+1}\right) \\
& = \Pr\left(  card_{X^{E^{m}}}=j\right)  \left(  1-\mu_{X}\left( e_{m+1}\right)  \right) 
 +\Pr\left(  card_{X^{E^{m}}}=j-1\right)  \mu _{X}\left(  e_{m+1}\right)
\end{align*}

Previous computations are summarized in expression \ref{EqAlgoritmoUnarioFA_1}%
. In algorithm \ref{alg:AlgoritmoUnarioFA}, the code for evaluating $\mathcal{F}%
^{A}\left(  Q\right)  \left(  X\right)  $ is presented. Complexity of the
algorithm is $O\left(  n^{2}\right)  $.%
\begin{align}
& \Pr\left(  card_{X}=j\right) \label{EqAlgoritmoUnarioFA_1}  =\left\{
\begin{array}
[c]{lll}%
\Pr\left(  card_{X^{E^{m}}}=0\right)  \left(  1-\mu_{X}\left(  e_{m+1}\right)
\right)  & : & j=0\\%
\Pr\left(  card_{X^{E^{m}}}=j\right)  \left(  1-\mu_{X}\left(  e_{m+1}\right) \right) & &\\
\; +\Pr\left(  card_{X^{E^{m}}}=j-1\right)  \mu_{X}\left(  e_{m+1}\right) & : & 1\leq j\leq m\\
\Pr\left(  card_{X^{E^{m}}}=m\right)  \mu_{X}\left(  e_{m+1}\right)  & : &
j=m+1
\end{array}
\right.  \nonumber%
\end{align}

\begin{algorithm}[!t]
\DontPrintSemicolon
\SetNoFillComment
\LinesNumbered
\linespread{1.0}
{ \footnotesize
\KwIn{The fuzzy set $X[0,\dots,m-1]$, $m \geq 1$, and a quantitative unary semi-fuzzy quantifier $q:\{0,\dots, m\} \rightarrow \textbf{I}$.}
\KwOut{The $result$ of the quantifier.}
\tcc{Assume all vector elements are initialized to zero}
$pr\_aux\_i \gets 0$\;
$pr\_aux\_i\_minus\_1 \gets 0$\;
$pr \gets [0,\dots,m]$\;
$result \gets 0$\;
$pr[0] \gets 1$\;
\For{$j \gets 0$; $j < m$; $j$\textit{++}} {
	$pr\_aux\_i \gets pr[0]$\;
	$pr[0] \gets (1 - X[j]) \times pr\_aux\_i$\;
	$pr\_aux\_i\_minus\_1 \gets pr\_aux\_i$\;
	\For{$i \gets 1$; $i \le j$; $i$\textit{++}} {
		$pr\_aux\_i \gets pr[i]$\;
		$pr[i] \gets (1 - X[j]) \times pr\_aux\_i + X[j] \times pr\_aux\_i\_minus\_1$\;
		$pr\_aux\_i\_minus\_1 \gets pr\_aux\_i$\;
	}
	$pr[j+1] \gets X[j] \times pr\_aux\_i\_minus\_1$\;
}
\For{$j \gets 0$; $j <= m$; $j$\textit{++}} {
	$result \gets result + pr[j] \times q(j)$\;
}
\KwRet{result}\;
}
\caption{Algorithm for computing unary quantitative quantifiers $\mathcal{F}^{A}(Q)(X)$}
\label{alg:AlgoritmoUnarioFA}
\end{algorithm}

\subsection{Conservative binary
quantifiers\label{ConservativeBinaryQuantifiers}}

In this section we will present the algorithm for evaluating conservative
binary quantifiers \cite{Keenan97VanBenthem}, which includes proportional
quantitative quantifiers as a particular case. The strategy we are going to
detail can be easily generalized for implementing other kinds of quantitative quantifiers.

A semi-fuzzy conservative quantitative quantifier $Q\left(  Y_{1}%
,Y_{2}\right)  $ depends on the cardinalities of $\left\vert Y_{1}\right\vert
$ and $\left\vert Y_{1}\cap Y_{2}\right\vert $; that is, there exists a
function $q:\left\{  0,\ldots,m\right\}  ^{2}\rightarrow\mathbf{I}$ such that:%
\[
Q\left(  Y_{1},Y_{2}\right)  =q\left(  \left\vert Y_{1}\right\vert ,\left\vert
Y_{1}\cap Y_{2}\right\vert \right)
\]
for all $Y_{1},Y_{2}\in\mathcal{P}\left(  E\right)  $.

Let $X_{1},X_{2}\in\mathcal{P}\left(  E\right)  $ be two fuzzy sets. By%
\begin{align*}
&  \Pr\left(  card_{X_{1},X_{1}\cap X_{2}}=\left(  j,k\right)  \right)   =\sum_{\substack{Y_{1}\in\mathcal{P}\left(  E\right) \; | \\ Y_{2}\in\mathcal{P}%
		\left(  E\right)}} 
	\sum_{\substack{\left\vert Y_{1}\right\vert =j \; \wedge \\ \left\vert Y_{1}\cap Y_{2}\right\vert =k}} m_{X_{1}}\left(  Y_{1}\right)  m_{X_{2}}\left(Y_{2}\right)  ,0\leq k,j\leq m
\end{align*}
we will denote the probability of choosing a pair of representatives
$Y_{1},Y_{2}\in\mathcal{P}\left(  E\right)  $ of $X_{1},X_{2}\in
\widetilde{\mathcal{P}}\left(  E\right)  $ such that$\left\vert Y_{1}%
\right\vert =j$ and $\left\vert Y_{1}\cap Y_{2}\right\vert =k$. It should be
noted that for $k>j$ $\Pr\left(  card_{X_{1},X_{1}\cap X_{2}}=\left(
j,k\right)  \right)  =0$.

Let $X_{1},X_{2}\in\widetilde{\mathcal{P}}\left(  E^{m}\right)  $ be two fuzzy
sets over $E^{m}=\left\{  e_{1},\ldots,e_{m}\right\}  $. And let us suppose we
know the probabilities:%
\[
\Pr\left(  card_{X_{1},X_{1}\cap X_{2}}=\left(  j,k\right)  \right)
\]
for all $j,k$ such that $0\leq j,k\leq m$. Let us suppose now we add an element
$e_{m+1}$ to the referential.\ That is, the new referential set is
$E^{m+1}=\left\{  e_{1},\ldots,e_{m+1}\right\}  $. And let $X_{1}^{\prime
},X_{2}^{\prime}\in\widetilde{\mathcal{P}}\left(  E^{m+1}\right)  $ be two
fuzzy sets in $E^{m+1}$ resulting of adding $e_{m+1}$. That is, $\left(
X_{1}^{\prime}\right)  ^{E^{m}}=X_{1},\left(  X_{2}^{\prime}\right)  ^{E^{m}%
}=X_{2}$; where by $\left(  {}\right)  ^{E^{m}}$ we are denoting the
projections of $X_{1}^{\prime},X_{2}^{\prime}$ over $E^{m}$.

By definition of the $\mathcal{F}^{A}$ DFS, belongniness of $e_{m+1}$ to the
set $X_{i}^{\prime},i=1,2$ is an event of probability $\mu_{X_{_{i}}^{\prime}%
}\left(  e_{m+1}\right)  $ and this probability is independent of the
belongniness of other elements. Then, if the cardinality of $X_{1},\,X_{2}%
$\ were $\left(  j,k\right)  $ and it would happen that $e_{m+1}\in
X_{1}^{\prime}$ and $e_{m+1}\in X_{2}^{\prime}$ then the cardinality of
$X_{1}^{\prime}\,X_{2}^{\prime}$ would be $\left(  j+1,k+1\right)  $.

Let $0\leq j,k\leq m$ be arbitray indexes and let us consider the probability
$\Pr(card_{X_{1},X_{1}\cap X_{2}}$ $=\left(  j,k\right)  )$.\ When we include
the element $e_{m+1}$ the probability of $e_{m+1}$ contributes to the
probability $\Pr\left(  card_{X_{1}^{\prime},X_{1}^{\prime}\cap X_{2}^{\prime
}}=\left(  j,k\right)  \right)  $ with:
\begin{align*}
&  \left(  1-\mu_{X_{1}^{\prime}}\left(  e_{m+1}\right)  \right)  \left(
1-\mu_{X_{2}^{\prime}}\left(  e_{m+1}\right)  \right) 
\Pr\left(
card_{X_{1},X_{1}\cap X_{2}}=\left(  j,k\right)  \right)  
+ \left(  1-\mu_{X_{1}^{\prime}}\left(  e_{m+1}\right)  \right)  \mu
_{X_{2}^{\prime}}\left(  e_{m+1}\right)  
\Pr\left(  card_{X_{1},X_{1}\cap
X_{2}}=\left(  j,k\right)  \right) \\
&  =\left(  1-\mu_{X_{1}^{\prime}}\left(  e_{m+1}\right)  \right)  \Pr\left(
card_{X_{1},X_{1}\cap X_{2}}=\left(  j,k\right)  \right)
\end{align*}
that is, if we know that the cardinality $card_{X_{1},X_{1}\cap X_{2}}$ is
$\left(  j,k\right)  $ then the cardinality $card_{X_{1}^{\prime}%
,X_{1}^{\prime}\cap X_{2}^{\prime}}$ would be $\left(  j,k\right)  $ with
probability $\left(  1-\mu_{X_{1}^{\prime}}\left(  e_{m+1}\right)  \right)  $.
It should be noted that if $e_{m+1}\notin X_{1}^{\prime}$ and $e_{m+1}\in
X_{2}^{\prime}$ then $e_{m+1}\notin X_{1}^{\prime}\cap X_{2}^{\prime}$.

Similarly, the contribution to $\Pr\left(  card_{X_{1}^{\prime},X_{1}^{\prime
}\cap X_{2}^{\prime}}=\left(  j+1,k\right)  \right)  $ will be:%
\[
\mu_{X_{1}^{\prime}}\left(  e_{m+1}\right)  \left(  1-\mu_{X_{2}^{\prime}%
}\left(  e_{m+1}\right)  \right)  \Pr\left(  card_{X_{1},X_{1}\cap X_{2}%
}=\left(  j,k\right)  \right)
\]
that is, as the cardinality $card_{X_{1},X_{1}\cap X_{2}}$ is $\left(
j,k\right)  $ then the cardinality $card_{X_{1}^{\prime},X_{1}^{\prime}\cap
X_{2}^{\prime}}$ will be $\left(  j+1,k\right)  $ with probability $\mu
_{X_{1}^{\prime}}\left(  e_{m+1}\right)  \left(  1-\mu_{X_{2}^{\prime}}\left(
e_{m+1}\right)  \right)  $.

And the contribution to the probability $\Pr\left(  card_{X_{1}^{\prime}%
,X_{1}^{\prime}\cap X_{2}^{\prime}}=\left(  j+1,k+1\right)  \right)  $ will
be:%
\[
\mu_{X_{1}^{\prime}}\left(  e_{m+1}\right)  \mu_{X_{2}^{\prime}}\left(
e_{m+1}\right)  \Pr\left(  card_{X_{1},X_{1}\cap X_{2}}=\left(  j,k\right)
\right)
\]
that is, as the cardinality $card_{X_{1},X_{1}\cap X_{2}}$ is $\left(
j,k\right)  $ then the cardinality $card_{X_{1}^{\prime},X_{1}^{\prime}\cap
X_{2}^{\prime}}$ will be $\left(  j+1,k+1\right)  $ with probability
$\mu_{X_{1}^{\prime}}\left(  e_{m+1}\right)  \mu_{X_{2}^{\prime}}e_{m+1}$.

Using previous expressions a polynomial algorithm can be developed to evaluate
conservative semi-fuzzy quantifiers (table \ref{alg:AlgoritmoConservativoFA}).
Complexity of the algorithm is $O\left(  n^{3}\right)  $.%

\begin{algorithm}[!t]
\DontPrintSemicolon
\SetNoFillComment
\LinesNumbered
\linespread{1.0}
{ \footnotesize
\KwIn{The fuzzy sets $X_{1}[0,\dots,m-1]$ and $X_{2}[0,\dots,m-1]$,  $m \geq 1$, and  a binary quantitative conservative semi-fuzzy quantifier $q:\mathbb{N}^{2} \rightarrow \mathbf{I}$.}
\KwOut{The $result$ of the quantifier.}
\tcc{Assume all vector elements are initialized to zero}
$card \gets [0, \dots, m][0, \dots, m]$\;
$card\_aux \gets [0, \dots, m][0, \dots, m]$\;
$result \gets 0$\;
$i \gets 0$\;
$card[0,0] \gets 1$\;
\While{$i < m$} {
	$clear(card\_aux)$\;
	$v\_i\_00\ \gets (1 - X_1[i]) \times (1 - X_2[i])$\;
	$v\_i\_01\ \gets (1 - X_1[i]) \times X_2[i]$\;
	$v\_i\_10\ \gets X_1[i] \times (1 - X_2[i])$\;
	$v\_i\_11 \gets X_1[i] \times X_2[i]$\;
	\For{$j \gets 0; j \le i; j$++} {
		\For{$k \gets 0; k \le j; k$++} {
			$card\_aux[j,k] \gets card\_aux[j,k] + (v\_i\_00 + v\_i\_01) \times card[j,k]$\;
			$card\_aux[j+1,k] \gets card\_aux[j+1,k] + v\_i\_10*card[j,k]$\;
			$card\_aux[j+1,k+1] \gets card\_aux[j+1,k+1] + v\_i\_11*card[j,k]$\;
		}
	}
	$copy(card\_aux, card)$\;	
	$i \gets i + 1$\;
}
\For{$j \gets 0; j \le m; j$++} {
	\For{$k \gets 0; k \le j; k$++} {
		$result \gets result + card[j,k] \times q(j,k)$\;
	}
}
\KwRet{result}\;
}
\caption{Algorithm for computing binary quantitative conservative quantifiers $\mathcal{F}^{A}\left( Q\right) \left( X_{1},X_{2}\right)$.}
\label{alg:AlgoritmoConservativoFA}
\end{algorithm}

It is not difficult to generalize the strategy we have presented to other
quantifiers. For example, let us consider the case of a ternary comparative
quantifier (e.g., \textit{\textquotedblleft the number of brilliant investors
that earn high salaries is about twice the number of brilliant investors that
earn low salaries\textquotedblright}). For this example, the semi-fuzzy
quantifier will follow the expression $Q\left(  Y_{1},Y_{2},Y_{3}\right)
=q\left(  \left\vert Y_{1}\cap Y_{2}\right\vert ,\left\vert Y_{1}\cap
Y_{3}\right\vert \right)  $ where $q:\left\{  0,\ldots,m\right\}
^{2}\rightarrow\mathbf{I}$ is the fuzzy number we use to model\ \textit{`about
twice'. }As this quantifier only depends on two boolean combinations, a binary
probability matrix will be enough to compute and update the probabilities of
the cardinalities. In this way, the complexity of the resulting algorithm will
be again $O\left(  n^{3}\right)  $.

In the general case, if a quantitative semi-fuzzy quantifier $Q:\mathcal{P}%
\left(  E\right)  ^{n}\longrightarrow\mathbf{I}$ depens on $r$ boolean
combinations its complexity will be $O\left(  m^{r+1}\right)  $,\ one
iteration to go through the input vectors and $r$ to go through the
probability matrix of cardinalities.

\subsection{Monte Carlo Approximation\label{MonteCarloApproximation}}

Monte Carlo simulation permits the approximation of the $\mathcal{F}^{A}$ QFM
when efficiency restrictions do not permit the use of the exact
implementations presented in previous sections. The idea of the Monte Carlo
simulation is simply to generate, for each $X_{i}$, a random binary vector
using a Bernoulli trial of probability $\mu_{X_{i}}\left(  j\right)  $ for
each $e_{j}$. The code for the Monte Carlo implementation of the
$\mathcal{F}^{A}$ model for an unary quantifier can be seen in table
\ref{alg:MonteCarloUnarioFA}. The extension to higher arity quantifiers is
trivial, by simply generating a random binary vector for each $X_{i}$. The
Monte Carlo approximation can be parallelized by simply dividing the number of
simulations between different processors.%


\begin{algorithm}[!t]
\DontPrintSemicolon
\SetNoFillComment
\LinesNumbered
\linespread{1.0}
{ \footnotesize
\KwIn{The fuzzy set $X[0,\dots,m-1]$, $m \geq 1$, the number of iterations $num\_simu$, and a quantitative unary semi-fuzzy quantifier $q:\{0,\dots, m\} \rightarrow \textbf{I}$.}
\KwOut{The $result$ of the quantifier.}
\tcc{Assume all vector elements are initialized to zero}
$x\_bino \gets [0,\dots,m]$\;
$v\_simu \gets [0,\dots,m]$\; 
\For{$i \gets 0$; $i < num\_simu$; $i$\textit{++}} {
	\For{$j \gets 0$; $j < m$; $j$\textit{++}} {
		\tcc{Simulation of a bernoulli trial with probability $X[j]$}
		$x\_bino[j] \gets \textit{BernouilliTrial}(X[j])$\; 
	}
	$v\_simu[Sum(x\_bino)]$++\;	
}
\tcc{Normalization to define a probability}
$v\_simu \gets v\_simu / num\_simu$\; 
$result \gets Sum(q^{T} \times v\_simu)$\;
\KwRet{result}\;
}
\caption{Monte Carlo Approximation for absolute unary quantifiers $\mathcal{F}^{A}\left( Q\right) \left( X\right)$.}
\label{alg:MonteCarloUnarioFA}
\end{algorithm}

\section{Conclusions}

In this paper we have presented several relevant results about the
$\mathcal{F}^{A}$ QFM. First, we summarized some of the most relevant
properties fulfilled by this model, in order to give a comprehensive and
integrative summary of its behavior. After that, we introduced a convergence result that guarantees that, in the limit case, the model
converges to the Zadeh's model for semi-fuzzy quantifiers defined by means of
proportional continuous fuzzy numbers. Moreover, this result is more
general than the specific convergence to the Zadeh's model, being applicable to
every proportional quantitative quantifier. For sufficiently
big fuzzy sets, this will allow to approximate the $\mathcal{F}^{A}$ QFM in
linear time.

However, the rate of convergence could be too slow to make this approximation
useful in most applications. For this reason, we also provided the exact
computational implementation for some of the most common quantifiers (unary
and proportional quantitative quantifiers), introducing a scheme that can be
easily extended to other types of quantifiers. Complexity of the exact
implementation is $O\left(  m^{r+1}\right)  $, being $r$ the number of boolean
combinations that are involved in the definition of the semi-fuzzy quantifier.

Finally, the convergence result has a strong implication. The underlying
probability of the $\mathcal{F}^{A}$ QFM will concentrate around the average
of the boolean combinations necessary to define the semi-fuzzy quantifier, as we increase the number of elements of the input
fuzzy sets. This property was used to propose a Monte Carlo approximation of
the $\mathcal{F}^{A}$ QFM that can be used when the complexity of the exact
implementation is too elevate to compute an exact solution.



\section*{Acknowledgment}This work has received financial support from the Consellería de Cultura, Educación e Ordenación Universitaria (accreditation 2016-2019, ED431G/08 and reference competitive group 2019-2021, ED431C 2018/29) and the European Regional Development Fund (ERDF) and is also supported by the Spanish Ministry of Economy and Competitiveness under the project TIN2015-73566-JIN.

\bibliographystyle{plain}
\bibliography{biblio}

\appendices
\section{Appendix: Computation of the probability of the projections of the
boolean combinations\label{AnnexProof}}

In this section we will develop an analyticial proof to show that the
probability projection $j$ of $f\left(  i_{1},\ldots,i_{K}\right)  $ follows a
binomial poisson distribution.

For developing the proof, we will need to introduce the complete definition of
some of the axioms of the DFS framework and of its derived properties.

The next definition allows the construction of a new semi-fuzzy quantifier
that simply permutes the arguments in the input:

\begin{definition}
[\textbf{Argument permutations}]\label{ArgPerm}\cite[Definition 4.13]%
{Glockner06Libro} Let $Q:\mathcal{P}\left(  E\right)  ^{n}\rightarrow
\mathbf{I}$ be a semi-fuzzy quantifier and $\beta:\left\{  1,\ldots,n\right\}
\rightarrow\left\{  1,\ldots,n\right\}  $ a permutation. By $Q\beta
:\mathcal{P}\left(  E\right)  ^{n}\rightarrow\mathbf{I}$ we denote the
semi-fuzzy quantifier defined by:%
\[
Q\beta\left(  Y_{1},\ldots,Y_{n}\right)  =Q\left(  Y_{\beta\left(  1\right)
},\ldots,Y_{\beta\left(  n\right)  }\right)
\]
for all $Y_{1},\ldots,Y_{n}\in\mathcal{P}\left(  E\right)  $. In the case of
fuzzy quantifiers $\widetilde{Q}\beta:\widetilde{\mathcal{P}}\left(  E\right)
^{n}\rightarrow\mathbf{I}$ is defined analogously.
\end{definition}

The next definition will allow us to rewrite permutations as a combination of transpositions:

\begin{definition}
[\textbf{Trasposition}]\cite[Definition 4.14]{Glockner06Libro}%
\label{DefPropTraspArg_1} For all $n\in\mathbb{N}$ $(n>0$) and $i,j\in\left\{
1,\ldots,n\right\}  $, the transposition $\tau_{i,j}:\left\{  1,\ldots
,n\right\}  \rightarrow\left\{  1,\ldots,n\right\}  $ is defined as:%
\[
\tau_{i,j}\left(  k\right)  =\left\{
\begin{tabular}
[c]{lll}%
$i$ & $:$ & $k=j$\\
$j$ & $:$ & $k=i$\\
$k$ & $:$ & $k\neq j\wedge k\neq i$%
\end{tabular}
\ \right.
\]
for all $k\in\left\{  1,\ldots,n\right\}  $. Moreover, by $\tau_{i}$ we will
denote the transposition $\tau_{i,n}$ (that interchanges positions $i$ and
$n$). It should be noted that $\tau_{i,j}=\tau_{i}\circ\tau_{j}\circ\tau_{i}$.
\end{definition}

We also can apply \ref{DefPropTraspArg_1} to fuzzy and semi-fuzzy quantifiers:

\begin{definition}
[\textbf{Argument transpositions}]Let $Q:\mathcal{P}\left(  E\right)
^{n}\rightarrow\mathbf{I}$ be a semi-fuzzy quantifier, $n>0$. By $Q\tau
_{i}:\mathcal{P}\left(  E\right)  ^{n}\rightarrow\mathbf{I}$ we denote the
semi-fuzzy quantifier defined by:
\begin{equation*}
Q\tau_{i}\left(  Y_{1},\ldots,Y_{i-1},Y_{i},Y_{i+1},\ldots,Y_{n}\right) \\
=Q\left(  Y_{1},\ldots,Y_{i-1},Y_{n},Y_{i+1},\ldots,Y_{i}\right)
\end{equation*}
for all $Y_{1},\ldots,Y_{n}\in\mathcal{P}\left(  E\right)  $. In the case of
semi-fuzzy quantifiers $\widetilde{Q}\tau_{i}:\widetilde{\mathcal{P}}\left(
E\right)  ^{n}\rightarrow\mathbf{I}$ is defined analogously.
\end{definition}

The DFS axiomatic framework guarantees the adequate generalization of argument transpositions:

\begin{theorem}
\cite[Theorem 4.16]{Glockner06Libro}\label{DefPropTrasArg}\textbf{\ }Every DFS
$\mathcal{F}$ is compatible with argument transpositions, i.e. for every
semi-fuzzy quantifier $Q:\mathcal{P}\left(  E\right)  ^{n}$ $\rightarrow
\mathbf{I}$ $i\in\left\{  1,\ldots,n\right\}  $,%
\[
\mathcal{F}\left(  Q\tau_{i}\right)  =\mathcal{F}\left(  Q\right)  \tau_{i}%
\]

\end{theorem}

Note that as permutations can be expressed as compositions of transpositions,
every DFS $\mathcal{F}$ also conmutes with permutations.

We will also need to introduce the full definition of external and internal negation.

\begin{definition}
[\textbf{External negation}]\cite[Definition 3.8]{Glockner06Libro} The external
negation of a semi-fuzzy quantifier $Q:\mathcal{P}\left(  E\right)
^{n}\rightarrow\mathbf{I}$ is defined by
\[
\left(  \widetilde{\lnot}Q\right)  \left(  Y_{1},\ldots,Y_{n}\right)
=\widetilde{\lnot}\left(  Q\left(  Y_{1},\ldots,Y_{n}\right)  \right)
\]
for all $Y_{1},\ldots,Y_{n}\in\mathcal{P}\left(  E\right)  $. The definition
of $\widetilde{\lnot}\widetilde{Q}:\widetilde{\mathcal{P}}\left(  E\right)
\rightarrow\mathbf{I}$ in the case of fuzzy quantifiers $\widetilde
{Q}:\widetilde{\mathcal{P}}\left(  E\right)  \rightarrow\mathbf{I}$ is
analogous\footnote{The reasonable choice of the fuzzy negation $\widetilde
{\lnot}:\mathbf{I}\rightarrow\mathbf{I}$ is the induced negation of the QFM.}.
\end{definition}

The next theorem expresses that every DFS correctly generalizes the external
negation property:

\begin{theorem}
\cite[Theorem 4.20]{Glockner06Libro}\label{DefPropNegExterna}\textbf{\ }Every
DFS $\mathcal{F}$\ is compatible with the formation of negations. Hence if
$Q:\mathcal{P}\left(  E\right)  ^{n}\rightarrow\mathbf{I}$ is a semi-fuzzy
quantifier then $\mathcal{F}\left(  \widetilde{\lnot}Q\right)  =\widetilde
{\lnot}\mathcal{F}\left(  Q\right)  $.
\end{theorem}

The internal negation or antonym of a semi-fuzzy quantifier is defined as:

\begin{definition}
[\textbf{Internal negation/antonym}]\cite[Definition 3.9]{Glockner06Libro} Let
a semi-fuzzy quantifier $Q:\mathcal{P}\left(  E\right)  ^{n}\rightarrow
\mathbf{I}$ of arity $n>0$ be given. The internal negation $Q\lnot
:\mathcal{P}\left(  E\right)  ^{n}\rightarrow\mathbf{I}$ of $Q$ is defined by
\[
Q\lnot\left(  Y_{1},\ldots,Y_{n}\right)  =Q\left(  Y_{1},\ldots,\lnot
Y_{n}\right)
\]
for all $Y_{1},\ldots,Y_{n}\in\mathcal{P}\left(  E\right)  $. The internal
negation $\widetilde{Q}\widetilde{\lnot}:\widetilde{\mathcal{P}}\left(
E\right)  ^{n}\rightarrow\mathbf{I}$ of a fuzzy quantifier $\widetilde
{Q}:\widetilde{\mathcal{P}}\left(  E\right)  ^{n}\rightarrow\mathbf{I}$ is
defined analogously, based on the given fuzzy complement $\widetilde{\lnot}$.
\end{definition}

The next theorem expresses that every DFS correctly generalizes the internal
negation property:

\begin{theorem}
\cite[Theorem 4.19]{Glockner06Libro}\label{DefPropNegInterna}\textbf{ }Every
DFS $\mathcal{F}$ is compatible with the negation of quantifiers. Hence if
$Q:\widetilde{\mathcal{P}}\left(  E\right)  ^{n}\rightarrow\mathbf{I}$ is a
semi-fuzzy quantifier, then $\mathcal{F}\left(  Q\lnot\right)  =\mathcal{F}%
\left(  Q\right)  \widetilde{\lnot}$.
\end{theorem}

We will now show the necessary definitions to establish the compatibility with
unions and intersections of quantifiers:

\begin{definition}
[\textbf{Union quantifier}]\cite[Definition 3.12]{Glockner06Libro} Let a
semi-fuzzy quantifier $Q:\mathcal{P}\left(  E\right)  ^{n}\rightarrow
\mathbf{I}$ of arity $n>0$ be given. We define the fuzzy quantifier
$Q\cup:\mathcal{P}\left(  E\right)  ^{n+1}\rightarrow\mathbf{I}$ as
\[
Q\cup\left(  Y_{1},\ldots,Y_{n},Y_{n+1}\right)  =Q\left(  Y_{1},\ldots
,Y_{n-1},Y_{n}\cup Y_{n+1}\right)
\]
for all $Y_{1},\ldots,Y_{n+1}\in\mathcal{P}\left(  E\right)  $. In the case of
fuzzy quantifiers $\widetilde{Q}\widetilde{\cup}$ is defined analogously,
based on a fuzzy definition of $\widetilde{\cup}$.
\end{definition}

Analogously, the definition of the intersection of quantifiers is:

\begin{definition}
[\textbf{Intersection quantifier}]Let $Q:\mathcal{P}\left(  E\right)
^{n}\rightarrow\mathbf{I}$ a semi-fuzzy quantifier, $n>0$, be given. We define
the semi-fuzzy quantifier $Q\cap:\mathcal{P}\left(  E\right)  ^{n+1}%
\rightarrow\mathbf{I}$ as%
\[
Q\cap\left(  Y_{1},\ldots,Y_{n},Y_{n+1}\right)  =Q\left(  Y_{1},\ldots
,Y_{n-1},Y_{n}\cap Y_{n+1}\right)
\]
for all $Y_{1},\ldots,Y_{n+1}\in\mathcal{P}\left(  E\right)  $. In the case of
fuzzy quantifiers $\widetilde{Q}\widetilde{\cap}$ is defined analogously,
based on a fuzzy definition of $\widetilde{\cap}$.
\end{definition}

\begin{theorem}
\cite[sections 3.9 and 4.9]{Glockner06Libro}\textbf{\ }Let $Q:\mathcal{P}%
\left(  E\right)  ^{n}\rightarrow\mathbf{I}$ be a semi-fuzzy quantifier,
$n>0$. Every DFS $\mathcal{F}$ is compatible with the union an intersection of
arguments%
\begin{align*}
\mathcal{F}\left(  Q\cup\right)   &  =\mathcal{F}\left(  Q\right)
\widetilde{\cup}\\
\mathcal{F}\left(  Q\cap\right)   &  =\mathcal{F}\left(  Q\right)
\widetilde{\cap}%
\end{align*}

\end{theorem}

Previous properties guarantee that arbitrary boolean combinations conmute
between fuzzy and semi-fuzzy quantifiers. This is one to the consequences of
the DFS axiomatic framework and it will be fundamental to prove that the
projection $j$ of $f\left(  i_{1},\ldots,i_{K}\right)  $ follows a binomial
poisson distribution.

\begin{example}
\label{ExampleBooleanCombination}Let $Q:\mathcal{P}\left(  E\right)
\rightarrow\mathbf{I}$ be a semi-fuzzy quantifier. And let $\Phi\left(
Y_{1},Y_{2},Y_{3}\right)  =\lnot Y_{1}\cap\lnot Y_{2}\cap Y_{3}$ be a boolean
combination of the crisp sets $Y_{1},Y_{2},Y_{3}\in\mathcal{P}\left(
E\right)  $, and $\Phi^{\prime}\left(  X_{1},X_{2},X_{3}\right)
=\widetilde{\lnot}X_{1}\widetilde{\cap}\widetilde{\lnot}X_{2}\widetilde{\cap
}X_{3}$ be the analogous boolean combination of fuzzy sets $X_{1},X_{2}%
,X_{3}\in\mathcal{P}\left(  E\right)  $ where $\widetilde{\lnot}%
,\widetilde{\cap}$ are defined by means of the corresponding negation and
tnorm induced by a particular DFS $\mathcal{F}$. Then\footnote{Notation used
in the example can result very confusing. For this reason, we will present
below the full detail $Q\cap\cap\tau_{1}\lnot\tau_{1}\tau_{2}\lnot\tau_{2}$,
explicitily detaling the application of the different transformations to the
semi-fuzzy quantifier:
\par
{}
\par%
\begin{align*}
\left(  Q\cap\cap\tau_{1}\lnot\tau_{1}\tau_{2}\lnot\tau_{2}\right)   
& =\left(  f^{\prime}:\left(  Y_{1}^{\prime},Y_{2}^{\prime}\right)  \rightarrow
Q\left(  Y_{1}^{\prime}\cap Y_{2}^{\prime}\right)  \right)  \cap\tau_{1}%
\lnot\tau_{1}\tau_{2}\lnot\tau_{2}\\
&  =\left(  f^{\prime\prime}:\left(  Y_{1}^{\prime\prime},Y_{2}^{\prime\prime
},Y_{3}^{\prime\prime}\right)  \rightarrow f^{\prime}\left(  Y_{1}%
^{\prime\prime},Y_{2}^{\prime\prime}\cap Y_{3}^{\prime\prime}\right)  \right)
\tau_{1}\lnot\tau_{1}\tau_{2}\lnot\tau_{2}\\
&  =\left(  f^{\prime\prime}:\left(  Y_{1}^{\prime\prime},Y_{2}^{\prime\prime
},Y_{3}^{\prime\prime}\right)  \rightarrow Q\left(  Y_{1}^{\prime\prime}\cap
Y_{2}^{\prime\prime}\cap Y_{3}^{\prime\prime}\right)  \right)  \tau_{1}%
\lnot\tau_{1}\tau_{2}\lnot\tau_{2}\\
&  =\left(  f^{\prime\prime\prime}:\left(  Y_{1}^{\prime\prime\prime}%
,Y_{2}^{\prime\prime\prime},Y_{3}^{\prime\prime\prime}\right)  \rightarrow
f^{\prime\prime}\left(  Y_{3}^{\prime\prime\prime},Y_{2}^{\prime\prime\prime
},Y_{1}^{\prime\prime\prime}\right)  \right)  \lnot\tau_{1}\tau_{2}\lnot
\tau_{2}\\
&  =\left(  f^{\prime\prime\prime}:\left(  Y_{1}^{\prime\prime\prime}%
,Y_{2}^{\prime\prime\prime},Y_{3}^{\prime\prime\prime}\right)  \rightarrow
Q\left(  Y_{3}^{\prime\prime\prime}\cap Y_{2}^{\prime\prime\prime}\cap
Y_{1}^{\prime\prime\prime}\right)  \right)  \lnot\tau_{1}\tau_{2}\lnot\tau
_{2}\\
&  =\left(  f^{\prime\prime\prime\prime}:\left(  Y_{1}^{\prime\prime
\prime\prime},Y_{2}^{\prime\prime\prime\prime},Y_{3}^{\prime\prime\prime
\prime}\right)  \rightarrow f^{\prime\prime\prime}\left(  Y_{1}^{\prime
\prime\prime\prime},Y_{2}^{\prime\prime\prime\prime},\lnot Y_{3}^{\prime
\prime\prime\prime}\right)  \right)  \tau_{1}\tau_{2}\lnot\tau_{2}\\
&  =\left(  f^{\prime\prime\prime\prime}:\left(  Y_{1}^{\prime\prime
\prime\prime},Y_{2}^{\prime\prime\prime\prime},Y_{3}^{\prime\prime\prime
\prime}\right)  \rightarrow Q\left(  \lnot Y_{3}^{\prime\prime\prime\prime
}\cap Y_{2}^{\prime\prime\prime\prime}\cap Y_{1}^{\prime\prime\prime\prime
}\right)  \right)  \tau_{1}\tau_{2}\lnot\tau_{2}\\
&  =\left(  f^{\prime\prime\prime\prime\prime}:\left(  Y_{1}^{\prime
\prime\prime\prime\prime},Y_{2}^{\prime\prime\prime\prime\prime},Y_{3}%
^{\prime\prime\prime\prime\prime}\right)  \rightarrow f^{\prime\prime
\prime\prime}\left(  Y_{3}^{\prime\prime\prime\prime\prime},Y_{2}%
^{\prime\prime\prime\prime\prime},Y_{1}^{\prime\prime\prime\prime\prime
}\right)  \right)  \tau_{2}\lnot\tau_{2}\\
&  =\left(  f^{\prime\prime\prime\prime\prime}:\left(  Y_{1}^{\prime
\prime\prime\prime\prime},Y_{2}^{\prime\prime\prime\prime\prime},Y_{3}%
^{\prime\prime\prime\prime\prime}\right)  \rightarrow Q\left(  \lnot
Y_{1}^{\prime\prime\prime\prime\prime}\cap Y_{2}^{\prime\prime\prime
\prime\prime}\cap Y_{3}^{\prime\prime\prime\prime\prime}\right)  \right)
\tau_{2}\lnot\tau_{2}\\
&  =\left(  f^{\prime\prime\prime\prime\prime\prime}:\left(  Y_{1}%
^{\prime\prime\prime\prime\prime\prime},Y_{2}^{\prime\prime\prime\prime
\prime\prime},Y_{3}^{\prime\prime\prime\prime\prime\prime}\right)  \rightarrow
f^{\prime\prime\prime\prime\prime}:\left(  Y_{1}^{\prime\prime\prime
\prime\prime\prime},Y_{3}^{\prime\prime\prime\prime\prime\prime},Y_{2}%
^{\prime\prime\prime\prime\prime\prime}\right)  \right)  \lnot\tau_{2}\\
&  =\left(  f^{\prime\prime\prime\prime\prime\prime}:\left(  Y_{1}%
^{\prime\prime\prime\prime\prime\prime},Y_{2}^{\prime\prime\prime\prime
\prime\prime},Y_{3}^{\prime\prime\prime\prime\prime\prime}\right)  \rightarrow
Q\left(  \lnot Y_{1}^{\prime\prime\prime\prime\prime\prime}\cap Y_{3}%
^{\prime\prime\prime\prime\prime\prime}\cap Y_{2}^{\prime\prime\prime
\prime\prime\prime}\right)  \right)  \lnot\tau_{2}\\
&  =\left(  f^{\prime\prime\prime\prime\prime\prime\prime}:\left(
Y_{1}^{\prime\prime\prime\prime\prime\prime\prime},Y_{2}^{\prime\prime
\prime\prime\prime\prime\prime},Y_{3}^{\prime\prime\prime\prime\prime
\prime\prime}\right)  \rightarrow f^{\prime\prime\prime\prime\prime\prime
}\left(  Y_{1}^{\prime\prime\prime\prime\prime\prime\prime},Y_{2}%
^{\prime\prime\prime\prime\prime\prime\prime},\lnot Y_{3}^{\prime\prime
\prime\prime\prime\prime\prime}\right)  \right)  \tau_{2}\\
&  =\left(  f^{\prime\prime\prime\prime\prime\prime\prime}:\left(
Y_{1}^{\prime\prime\prime\prime\prime\prime\prime},Y_{2}^{\prime\prime
\prime\prime\prime\prime\prime},Y_{3}^{\prime\prime\prime\prime\prime
\prime\prime}\right)  \rightarrow Q\left(  \lnot Y_{1}^{\prime\prime
\prime\prime\prime\prime\prime}\cap\lnot Y_{3}^{\prime\prime\prime\prime
\prime\prime\prime}\cap Y_{2}^{\prime\prime\prime\prime\prime\prime\prime
}\right)  \right)  \tau_{2}\\
&  =\left(  f^{\prime\prime\prime\prime\prime\prime\prime\prime}:\left(
Y_{1}^{\prime\prime\prime\prime\prime\prime\prime\prime},Y_{2}^{\prime
\prime\prime\prime\prime\prime\prime\prime},Y_{3}^{\prime\prime\prime
\prime\prime\prime\prime\prime}\right)  \rightarrow f^{\prime\prime
\prime\prime\prime\prime\prime}\left(  Y_{1}^{\prime\prime\prime\prime
\prime\prime\prime\prime},Y_{3}^{\prime\prime\prime\prime\prime\prime
\prime\prime},Y_{2}^{\prime\prime\prime\prime\prime\prime\prime\prime}\right)
\right) \\
&  =\left(  f^{\prime\prime\prime\prime\prime\prime\prime\prime}:\left(
Y_{1}^{\prime\prime\prime\prime\prime\prime\prime\prime},Y_{2}^{\prime
\prime\prime\prime\prime\prime\prime\prime},Y_{3}^{\prime\prime\prime
\prime\prime\prime\prime\prime}\right)  \rightarrow Q\left(  \lnot
Y_{1}^{\prime\prime\prime\prime\prime\prime\prime\prime}\cap\lnot
Y_{2}^{\prime\prime\prime\prime\prime\prime\prime\prime}\cap Y_{3}%
^{\prime\prime\prime\prime\prime\prime\prime\prime}\right)  \right)
\end{align*}
}%
\begin{align*}
\mathcal{F}\left(  Q\circ\Phi\right)  \left(  X_{1},X_{2},X_{3}\right)   
& =\mathcal{F}\left(  Q\cap\cap\tau_{1}\lnot\tau_{1}\tau_{2}\lnot\tau
_{2}\right)  \left(  X_{1},X_{2},X_{3}\right) \\
&  =\mathcal{F}\left(  Q\right)  \widetilde{\cap}\widetilde{\cap}\tau
_{1}\widetilde{\lnot}\tau_{1}\tau_{2}\widetilde{\lnot}\tau_{2}\left(
X_{1},X_{2},X_{3}\right) \\
& =\left(  \mathcal{F}\left(  Q\right)  \circ\Phi^{\prime}\right)  \left(
X_{1},X_{2},X_{3}\right) \\
& =\mathcal{F}\left(  Q\right)  \left(  \widetilde
{\lnot}X_{1}\widetilde{\cap}\widetilde{\lnot}X_{2}\widetilde{\cap}%
X_{3}\right)
\end{align*}

\end{example}

Now, we will introduce some notation to specify that the cardinality of the
input sets of a semi-fuzzy quantifiers is exactly of `$i$ elements':

\begin{notation}
We will denote by $q_{exactly}^{i}:\left\{  0,\ldots,m\right\}
\longrightarrow\left\{  0,1\right\}  $ the function defined by%
\[
q_{exactly}^{i}\left(  x\right)  =\left\{
\begin{tabular}
[c]{lll}%
$1$ & $:$ & $x=i$\\
$0$ & $:$ & $otherwise$%
\end{tabular}
\ \right.
\]
and by $Q_{exactly}^{i,n}:\mathcal{P}\left(  E\right)  ^{n}\longrightarrow
\mathbf{I}$ the semi-fuzzy quantifier defined as:%
\[
Q_{exactly}^{i,n}\left(  Y_{1},\ldots,Y_{n}\right)  =q_{exactly}^{i}\left(
\left\vert Y_{1}\cap\ldots\cap Y_{n}\right\vert \right)
\]
where with the superindex $n$ we are indicating the arity of the semi-fuzzy quantifier.
\end{notation}

\begin{proposition}
\label{PropEqBinFA}Let $X^{\mathbf{B}}\in\widetilde{\mathcal{P}}\left(
E\right)  $ a fuzzy set where $\mathbf{B}=P_{1},\ldots,P_{m}$ is its
corresponding poisson bernoulli succession with probabilities $p_{1}%
,\ldots,p_{m}$. It is fulfilled:%
\[
\Pr^{\mathbf{B}}\left(  K=k\right)  =\mathcal{F}^{A}\left(  Q_{exactly}%
^{i,1}\right)  \left(  X^{\mathbf{B}}\right)
\]

\end{proposition}

\begin{proof}
Simply:%
\begin{align*}
\mathcal{F}^{A}\left(  Q_{exactly}^{i}\right)  \left(  X^{\mathbf{B}}\right) 
& \qquad =\sum_{Y\in\mathcal{P}\left(  E\right)  }m_{X^{\mathbf{B}}}\left(
Y\right)  Q_{exactly}^{i,1}\left(  Y\right) \\
& \qquad  =\sum_{Y\in\mathcal{P}\left(  E\right)  |\left\vert Y\right\vert
=i}m_{X^{\mathbf{B}}}\left(  Y\right) \\
& \qquad  =\sum_{Y\in\mathcal{P}\left(  E\right)  |\left\vert Y\right\vert =i}%
{\displaystyle\prod\limits_{e\in Y}}
\mu_{X^{\mathbf{B}}}\left(  e\right)
{\displaystyle\prod\limits_{e\in Y^{c}}}
\left(  1-\mu_{X^{\mathbf{B}}}\left(  e\right)  \right) \\
& \qquad  =\sum_{A\in F_{k}}%
{\displaystyle\prod\limits_{i\in A}}
p_{i}%
{\displaystyle\prod\limits_{j\in A^{c}}}
\left(  1-p_{j}\right) \\
& \qquad  =\Pr^{\mathbf{B}}\left(  K=k\right)
\end{align*}

\end{proof}

And before proceeding to the main proof of this section, we need to introduce
the following lemma:

\begin{lemma}
\label{LemaBooleanCombinations}Let $X_{1},\ldots,X_{n}\in\widetilde
{\mathcal{P}}\left(  E\right)  $ and $\Phi_{l_{1},\ldots,l_{n}}\left(
X_{1},\ldots,X_{n}\right)  =X_{1}^{\left(  l_{1}\right)  }\widetilde{\cap
}\ldots\widetilde{\cap}X_{n}^{\left(  l_{n}\right)  }$\footnote{Following
notation in \ref{TeoremaCuantitativo}, by $X_{1}^{\left(  l_{1}\right)  }$ we
are denoting the set $X_{1}$ in case $l_{1}=1$ and $\widetilde{\lnot}X_{1}$ in
case $l_{1}=0$.} a boolean combination of $X_{1},\ldots,X_{n}$, then it is
fulfilled:%
\begin{align*}
\mathcal{F}^{A}\left(  Q_{exactly}^{i_{j},1}\right)  \left(  X_{1}^{\left(
l_{1}\right)  }\widetilde{\cap}\ldots\widetilde{\cap}X_{n}^{\left(
l_{n}\right)  }\right)  
& =\sum_{Y\in\mathcal{P}\left(  E\right) \; | \; \left\vert
Y\right\vert =j}m_{X_{1}^{\left(  l_{1}\right)  }\widetilde{\cap}%
\ldots\widetilde{\cap}X_{n}^{\left(  l_{n}\right)  }}\left(  Y\right)  \\
& =\sum_{\substack{Y_{1},\ldots,Y_{n}\in\mathcal{P}\left(  E\right) \; | \\
		\left\vert Y_{1} \cap\ldots\cap Y_{n}\right\vert =j}} m_{X_{1}^{\left(  l_{1}\right)  }}\left(
Y_{1}\right)  \ldots m_{X_{n}^{\left(  l_{n}\right)  }}\left(  Y_{n}\right) \\
& =\mathcal{F}^{A}\left(  Q_{exactly}^{i_{j},1}\circ\Phi_{l_{1},\ldots,l_{n}%
}\right)  \left(  X_{1},\ldots,X_{n}\right)
\end{align*}
\end{lemma}%

\begin{proof}
Being $\mathcal{F}^{A}$ a DFS, we have seen it conmutes with boolean
combinations. Then:%
\begin{align*}
\mathcal{F}^{A}\left(  Q_{exactly}^{i_{j},1}\right)  \left(  X_{1}^{\left(
l_{1}\right)  }\widetilde{\cap}\ldots\widetilde{\cap}X_{n}^{\left(
l_{n}\right)  }\right)  
& =\sum_{Y\in\mathcal{P}\left(  E\right)  } m_{X_{1}^{\left(  l_{1}\right)  }\widetilde{\cap}\ldots\widetilde{\cap}%
X_{n}^{\left(  l_{n}\right)  }}\left(  Y\right)  Q_{exactly}^{j,1}\left(
Y\right) \\
& =\sum_{Y\in\mathcal{P}\left(  E\right) \; | \; \left\vert Y\right\vert
=j}m_{X_{1}^{\left(  l_{1}\right)  }\widetilde{\cap}\ldots\widetilde{\cap
}X_{n}^{\left(  l_{n}\right)  }}\left(  Y\right) \\
& =\mathcal{F}^{A}\left(  Q_{exactly}^{j,1}\right)  \left(  \Phi
_{l_{1},\ldots,l_{n}}\left(  X_{1},\ldots,X_{n}\right)  \right) \\
& =\mathcal{F}^{A}\left(  Q_{exactly}^{j,1}\circ\Phi_{l_{1},\ldots,l_{n}%
}\right)  \left(  X_{1},\ldots,X_{n}\right)
\end{align*}

\end{proof}

and let $j_{1},\ldots,j_{s}\in\left\{  1,\ldots,n\right\}  $ the ordered set
of indexes in $l_{1},\ldots,l_{n}$ such that $l_{j_{s}}=0$ (i.e., the ones
that complement the input argument). Example \ref{ExampleBooleanCombination}
showed that $\Phi_{l_{1},\ldots,l_{n}}$ is of the form $\underset{n}%
{\cap\ldots\cap}\tau_{j_{1}}\lnot\tau_{j_{1}}\tau_{j_{2}}\lnot\tau_{j_{2}%
}\ldots\tau_{j_{s}}\lnot\tau_{j_{s}}$. That is:%
\begin{equation*}
\Phi_{l_{1},\ldots,l_{n}}\left(  Y_{1},\ldots,Y_{n}\right)  \\
=\underset{n}%
{\cap\ldots\cap}\tau_{j_{1}}\lnot\tau_{j_{1}}\tau_{j_{2}}\lnot\tau_{j_{2}%
}\ldots\tau_{j_{s}}\lnot\tau_{j_{s}}\left(  Y_{1},\ldots,Y_{n}\right)
\end{equation*}
then
\begin{align*}
\mathcal{F}^{A}\left(  Q_{exactly}^{i_{j},1}\right)  \left(  X_{1}^{\left(
l_{1}\right)  }\widetilde{\cap}\ldots\widetilde{\cap}X_{n}^{\left(
l_{n}\right)  }\right) 
&  =\ldots\\
&  =\mathcal{F}^{A}\left(  Q_{exactly}^{j,1}\underset{n}{\cap\ldots\cap}%
\tau_{j_{1}}\lnot\tau_{j_{1}}\tau_{j_{2}}\lnot\tau_{j_{2}}\ldots\tau_{j_{s}%
}\lnot\tau_{j_{s}}\right) \\
& \qquad  \left(  X_{1},\ldots,X_{n}\right) \\
&  =\mathcal{F}^{A}\left(  Q_{exactly}^{j,1}\underset{n}{\cap\ldots\cap
}\right)  \left(  X_{1}^{\left(  l_{1}\right)  },\ldots,X_{n}^{\left(
l_{n}\right)  }\right) \\
&  =\sum_{Y_{1}\in\mathcal{P}\left(  E\right)  }\ldots\sum_{Y_{n}%
\in\mathcal{P}\left(  E\right)  }m_{X_{1}^{\left(  l_{1}\right)  }}\left(
Y_{1}\right)  \ldots m_{X_{n}^{\left(  l_{n}\right)  }}\left(  Y_{n}\right) \\
& \qquad \left(  Q_{exactly}^{j,1}\underset{n}{\cap\ldots\cap}\right)  \left(
Y_{1},\ldots,Y_{n}\right) \\
&  =\sum_{\substack{Y_{1},\ldots,Y_{n}\in\mathcal{P}\left(  E\right) \; | \\ \left\vert
Y_{1}\cap\ldots\cap Y_{n}\right\vert =j}} m_{X_{1}^{\left(  l_{1}\right)  }%
}\left(  Y_{1}\right)  \ldots m_{X_{n}^{\left(  l_{n}\right)  }}\left(
Y_{n}\right)
\end{align*}
\medskip

Now, we will compute the projection of the probability function for
quantitative quantifiers.

By theorem \ref{TeoremaCuantitativo} every semi-fuzzy quantifier
$Q:\mathcal{P}\left(  E\right)  ^{n}\longrightarrow\mathbf{I}$ can be defined
by means of a function $q:\left\{  0,\ldots,\left\vert E\right\vert \right\}
^{K}\longrightarrow\mathbf{I}$ depending on the\ cardinalities of the boolean
combinations of the input sets ($\left\vert \Phi_{1}\left(  Y_{1},\ldots
,Y_{n}\right)  \right\vert ,\ldots,\left\vert \Phi_{K}\left(  Y_{1}%
,\ldots,Y_{n}\right)  \right\vert $). Then:%

\begin{align*}
\mathcal{F}^{A}\left(  Q\right)  \left(  X_{1},\ldots,X_{n}\right) 
&  =\sum_{Y_{1}\in\mathcal{P}\left(  E\right)  }\ldots\sum_{Y_{n}%
\in\mathcal{P}\left(  E\right)  }m_{X_{1}}\left(  Y_{1}\right)  \ldots
m_{X_{n}}\left(  Y_{n}\right) Q\left(  Y_{1},\ldots,Y_{n}\right) \\
&  =\sum_{\left(  i_{1},\ldots,i_{K}\right)  \in m^{K}}\sum_{\substack{Y_{1}%
,\ldots,Y_{n}\in\mathcal{P}\left(  E\right) \; | \\\left\vert \Phi_{1}\left(
Y_{1},\ldots,Y_{n}\right)  \right\vert =i_{1}\wedge\\\ldots\\\left\vert
\Phi_{k}\left(  Y_{1},\ldots,Y_{n}\right)  \right\vert =i_{K}}}m_{X_{1}%
}\left(  Y_{1}\right)  \ldots m_{X_{n}}\left(  Y_{n}\right)  q\left(
\left\vert \Phi_{1}\left(  Y_{1},\ldots,Y_{n}\right)  \right\vert
,\ldots,\left\vert \Phi_{K}\left(  Y_{1},\ldots,Y_{n}\right)  \right\vert
\right) \\
&  =\sum_{\left(  i_{1},\ldots,i_{K}\right)  \in m^{K}} q\left(  i_{1},\ldots,i_{K}\right)  
\sum_{\mathclap{\substack{Y_{1},\ldots,Y_{n}\in\mathcal{P}\left(
E\right)  |\\\left\vert \Phi_{1}\left(  Y_{1},\ldots,Y_{n}\right)  \right\vert
=i_{1}\wedge\\\ldots\\\left\vert \Phi_{K}\left(  Y_{1},\ldots,Y_{n}\right)
\right\vert =i_{K}}}} \; m_{X_{1}}\left(  Y_{1}\right)  \ldots m_{X_{n}}\left(
Y_{n}\right)
\end{align*}

Let us denote by%
\begin{equation}
f\left(  i_{1},\ldots,i_{K}\right)  =\sum_{\substack{Y_{1},\ldots,Y_{n}%
\in\mathcal{P}\left(  E\right) \; |\\\left\vert \Phi_{1}\left(  Y_{1}%
,\ldots,Y_{n}\right)  \right\vert =i_{1}\wedge\\\ldots\\\left\vert \Phi
_{K}\left(  Y_{1},\ldots,Y_{n}\right)  \right\vert =i_{K}}}m_{X_{1}}\left(
Y_{1}\right)  \ldots m_{X_{n}}\left(  Y_{n}\right)  \label{EqProbabilityMx2}%
\end{equation}

$f\left(  i_{1},\ldots,i_{K}\right)  $ is a probability function. Take into
account that $m_{X_{1}}\left(  Y_{1}\right)  \ldots m_{X_{n}}\left(
Y_{n}\right)  $ define a probability over $\left(  Y_{1},\ldots,Y_{n}\right)
\in\mathcal{P}\left(  E\right)  ^{n}$, and $f\left(  i_{1},\ldots
,i_{K}\right)  $ simply distributes the probabilities on $\left(  Y_{1}%
,\ldots,Y_{n}\right)  \in\mathcal{P}\left(  E\right)  ^{n}$ over the
cardinalities of the $K$ boolean combinations.

\begin{theorem}
Let $f\left(  i_{1},\ldots,i_{K}\right)  $ be the probability distribution
that is obtained when we compute the probability induced by $X_{1}%
,\ldots,X_{n}\in\mathcal{P}\left(  E\right)  ^{n}$ fuzzy sets over the
cardinalities of the boolean combinations $\left\vert \Phi_{1}\left(
Y_{1},\ldots,Y_{n}\right)  \right\vert ,\ldots,\left\vert \Phi_{K}\left(
Y_{1},\ldots,Y_{n}\right)  \right\vert $. The probability projection $j$ of
$f\left(  i_{1},\ldots,i_{K}\right)  $ will follow a binomial poisson
distribution of parameters%
\begin{align*}
p_{1}^{j}  &  =\mu_{X_{1}^{\left(  l_{j,1}\right)  }\widetilde{\cap}%
\ldots\widetilde{\cap}X_{n}^{\left(  l_{j,n}\right)  }}\left(  e_{1}\right) \\
&  \ldots\\
p_{m}^{j}  &  =\mu_{X_{1}^{\left(  l_{j,1}\right)  }\widetilde{\cap}%
\ldots\widetilde{\cap}X_{n}^{\left(  l_{j,n}\right)  }}\left(  e_{m}\right)
\end{align*}
where $X_{1}^{\left(  l_{j,1}\right)  }\widetilde{\cap}\ldots\widetilde{\cap
}X_{n}^{\left(  l_{j,n}\right)  }=\Phi_{j}\left(  X_{1},\ldots,X_{n}\right)  $
\end{theorem}

\begin{proof}
Let us consider the projection $j$ of $f\left(  i_{1},\ldots,i_{K}\right)  $.
Using the same ideas than in the proof of lemma \ref{LemaBooleanCombinations}:%
\begin{align*}
f^{j}\left(  i_{j}\right)   &  =\sum_{i_{1},\ldots,i_{j-1},i_{j+1}%
,\ldots,i_{K}}f\left(  i_{1},\ldots,i_{j},\ldots,i_{K}\right) \\
&  =\sum_{\substack{Y_{1},\ldots,Y_{n}\in\mathcal{P}\left(  E\right)
|\\\left\vert \Phi_{j}\left(  Y_{1},\ldots,Y_{n}\right)  \right\vert =i_{j}%
}}m_{X_{1}}\left(  Y_{1}\right)  \ldots m_{X_{n}}\left(  Y_{n}\right) \\
&  =\sum_{Y_{1},\ldots,Y_{n}\in\mathcal{P}\left(  E\right)  }m_{X_{1}}\left(
Y_{1}\right)  \ldots m_{X_{n}}\left(  Y_{n}\right)  \\
& \qquad \left(  Q_{exactly}^{i_{j},1}\circ\Phi_{j}\right)  \left(  Y_{1},\ldots,Y_{n}\right) \\
&  =\mathcal{F}^{A} \left(  Q_{exactly}^{i_{j},1}\circ\Phi_{j}\right)  \left(
X_{1},\ldots,X_{n}\right) \\
&  =\sum_{Y\in\mathcal{P}\left(  E\right)  |\left\vert Y\right\vert
=j}m_{X_{1}^{\left(  l_{1}\right)  }\widetilde{\cap}\ldots\widetilde{\cap
}X_{n}^{\left(  l_{n}\right)  }}\left(  Y\right)
\end{align*}

but this is a binomial poisson bernoulli succession with distribution
$\mathbf{B}=P_{1},\ldots,P_{m}$ with probabilities (proposition
\ref{PropEqBinFA})%
\begin{align*}
p_{1}  &  =\mu_{X_{1}^{\left(  l_{1}\right)  }\widetilde{\cap}\ldots
\widetilde{\cap}X_{n}^{\left(  l_{n}\right)  }}\left(  e_{1}\right) \\
&  \ldots\\
p_{m}  &  =\mu_{X_{1}^{\left(  l_{1}\right)  }\widetilde{\cap}\ldots
\widetilde{\cap}X_{n}^{\left(  l_{n}\right)  }}\left(  e_{m}\right)
\end{align*}

\end{proof}

%

\end{document}